\documentclass{article}

\usepackage{microtype}
\usepackage{booktabs} 

\usepackage[utf8]{inputenc}
\usepackage{latexsym}         
\usepackage{amsmath}          
\usepackage{amssymb}          
\usepackage{graphicx}
\usepackage{algorithm}
\usepackage{algorithmic}
\usepackage{hyperref}

\usepackage{tcolorbox}
\usepackage{enumerate}
\usepackage{subfig}
\usepackage{amsthm}
\usepackage{thmtools}
\usepackage{thm-restate}
\usepackage[inline]{enumitem}
\usepackage{tasks}
\setlist[enumerate,1]{leftmargin=*,wide=0em, noitemsep,nolistsep, label = {\textbf{\arabic.}}}
\setlist[itemize,1]{leftmargin=*,wide=0em, noitemsep,nolistsep}

\newenvironment{proof-sketch}{\noindent{\textit{Sketch of Proof.}}\hspace*{1em}}{\qed\bigskip}
\newtheorem{lemma}{Lemma}
\newtheorem{theo}{Theorem}
\newtheorem{prop}{Proposition}
\newtheorem{defi}{Definition}

\newtheorem{ass}{Assumption}
\newtheorem{remark}{Remark}

\newcommand{\de}{\mathrm{d}}

\DeclareMathOperator*{\argmin}{argmin}

\usepackage[accepted]{icml2021}

\begin{document}

\twocolumn[
\icmltitle{Quantile Bandits for Best Arms Identification}



\icmlsetsymbol{equal}{*}

\begin{icmlauthorlist}
\icmlauthor{Mengyan Zhang}{ANU,Data61}
\icmlauthor{Cheng Soon Ong}{Data61,ANU}
\end{icmlauthorlist}

\icmlaffiliation{ANU}{The Australian National University}
\icmlaffiliation{Data61}{Data61, CSIRO}

\icmlcorrespondingauthor{Cheng Soon Ong}{chengsoon.ong@anu.edu.au}

\icmlkeywords{Machine Learning, ICML}

\vskip 0.3in
]



\printAffiliationsAndNotice{}  

\begin{abstract}


    We consider a variant of the best arm identification task in stochastic multi-armed bandits.
    Motivated by risk-averse decision-making problems, our goal is to identify a set of $m$ arms with the highest $\tau$-quantile values within a fixed budget.
    We prove asymmetric two-sided concentration inequalities for order statistics and quantiles of random variables that have non-decreasing hazard rate, which may be of independent interest.
    With these inequalities, we analyse a quantile version of Successive Accepts and Rejects (Q-SAR).
    We derive an upper bound for the probability of arm misidentification, the first justification of a quantile based algorithm for fixed budget multiple best arms identification.
    We show illustrative experiments for best arm identification.
    \end{abstract}

    \section{Introduction}
    \label{sec: Introduction}

    Multi-Armed Bandits (MAB) are sequential experimental design problems
    where an agent adaptively chooses one (or multiple) option(s) among a set of choices based on certain policies.
    We refer to ``options'' as ``arms'' in MAB problems.
    In contrast with \textit{full feedback} online decision-making problems where sample rewards for all arms are fully observable to agents in each round,
    in MAB tasks the agent only observes the sample reward from the selected arm in each round, with no information about other arms.

    One of the key steps in the theoretical analysis for bandit algorithms is concentration inequalities,
    which provides bounds on how a random variable deviates from some statistical summary (typically its expected value).
    Inspired by the approach in~\citet{boucheron2012},
    we propose in Section~\ref{sec: concentration inequalities} new concentration inequalities
    for order statistics and quantiles
    of distributions with non-decreasing hazard rates (Definition \ref{defi: Hazard Rate}).
    Previous work derived concentration bounds of quantiles via the empirical cumulative distribution function (c.d.f.).
    Our proof uses a new approach based on the extended Exponential Efron-Stein inequality (Theorem \ref{theo: New Extended Exponential Efron-Stein inequality}),
    and non-trivially extends the concentration of order statistics~\citep{boucheron2012}.
    Our proposed concentration inequality can be useful for various applications,
    for example, the multi-armed bandits problem as illustrated in this work,
    learning theory, A/B-testing \citep{howard_sequential_2019}, and model selection \citep{massart2000some}.

    \begin{table}[t]
      \centering
      \resizebox{0.4\textwidth}{!}{
        \begin{tabular}{|c|c|c|c|c|c|}
          & Mean   & 0.5-Quantile &   & Mean & 0.8-Quantile\\
    A     &   3.50  & \bf{3.50} & C  & 1.45  & 2.33\\
    B     & \bf{4.00}  &  2.80  & D  & \bf{2.50}  & \bf{4.00}\\
    OptArm & B      & A & Gap &1.05      &   1.67 \\
    \end{tabular}
      }
      \caption{Summary statistics for toy example rewards}%
      \label{table: Median and mean for simulated rewards}
    \end{table}

    \begin{figure}[t]
      \centering
      \includegraphics[scale=0.4]{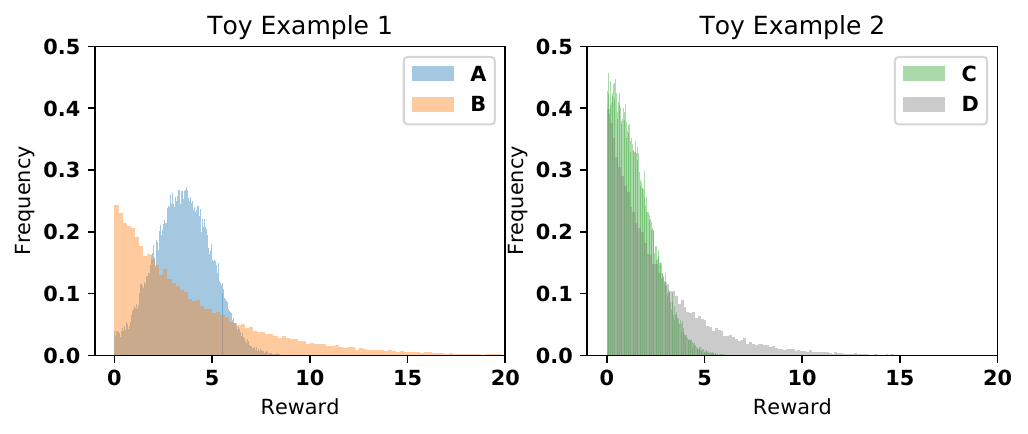}
      \caption{Toy example reward histograms.
        }
        \label{fig: Simulated reward histogram.}
    \end{figure}

    We apply the proposed concentration inequality to the \textit{Best Arm Identification} (BAI) task with \textit{fixed budget} \citep{audibert2010best,bubeck2013multiple}.
    The goal of BAI is to select the best arm (in our case the top $m \geq 1$ arms) after the exploration phase (i.e. budget has run out).
    The agent can explore the environment and perform actions during the exploration phase without penalty.
    In contrast to the majority of previous work which identifies optimal arms by summarising a distribution by its mean,
    we address risk-averse bandits by evaluating the quality of arms by a quantile value of the reward distribution.
    We consider the bandit problem of quantile based $m$ best arms identification,
    where the goal is to identify a set of $m$ arms with the highest $\tau$-quantile values.
    To the best of our knowledge, existing quantile work focuses on single optimal arm identification,
    and we are the first work that addresses multiple best arms identification
    for fixed budget setting with respect to the $\tau$-quantile.
    Our proposed algorithm is in Section~\ref{sec: quantile bandits q-sar}.

    Studying quantile concentrations and identifying arms with optimal quantiles have been shown to be beneficial for many cases, such as
    when the rewards are qualitative \citep{szorenyi2015qualitative},
    when the decision-making is risk-averse \citep{yu2013sample, david_pac_nodate},
    when the rewards contain outliers
    \citep{altschuler2019best}, or
    when the reward distributions are highly skewed \citep{howard_sequential_2019}.

    We motivate the use of quantiles as summary statistics with two toy examples,
    with simulated reward histograms of 2 arms shown in Figure \ref{fig: Simulated reward histogram.}
    and corresponding summary statistics shown in Table \ref{table: Median and mean for simulated rewards}.
    The first example illustrates when risk-averse agents should prefer quantiles.
    Consider a vaccine testing problem \citep{cunningham_vaccine_2016}, where the goal is to identify the most reliable vaccine after the exploration phase.
    The reward is the efficacy of the vaccine.
    Risk-averse agents tend to exclude vaccine candidates which return a large number of small rewards
    even though they may have a larger expected value (e.g. B in Figure \ref{fig: Simulated reward histogram.}).
    In such case, a policy guided by a fixed level of quantiles (e.g. 0.5-quantile, the median) will choose a less risky arm (i.e. one with less low rewards).
    The second example shows when the distributions are skewed,
    the quantile can provide a bigger gap between arms, which turns out to produce a smaller probability of error (Definition \ref{defi: probability of error}).
    As shown by toy example 2 in Figure \ref{fig: Simulated reward histogram.}, the quantile and mean reflect the same preference, but the difference between arms is larger for the 0.8-quantile.
    The choice of quantile level $\tau$ provides an extra degree of modelling freedom, that the practitioner may use to capture domain requirements
    or to achieve a smaller error probability.


    Our \textbf{contributions} are:
    \begin{enumerate*}[mode=unboxed,topsep=0pt,itemsep=-1ex,partopsep=1ex,parsep=1ex,label = \textbf{\textcolor{blue}{(\roman*)}}]
    \item Two-sided exponential concentration inequalities for order statistics of rank $k$ (w.r.t its expectation) for a general family (with non-decreasing hazard rate) of random variables.
    \item Two-sided exponential concentration inequalities for estimations of $\tau$-quantile (w.r.t. population quantile) based on our results on order statistics.
    \item The first $\tau$-quantile based multiple ($m\geq 1$) arms identification algorithm (Q-SAR) for the fixed budget setting.
    \item Theoretical analysis for the proposed Q-SAR algorithm, showing an exponential rate upper bound on the probability of error.
    \item Empirical illustrations for the Q-SAR algorithm, which indicates that Q-SAR outperforms baseline algorithms for the best arms identifications task.
    \end{enumerate*}

\section{Concentration Inequalities}
    \label{sec: concentration inequalities}

    In this section, we show our results for concentration inequalities on order statistics and quantiles.
    We apply these results to prove error bounds for bandits in Section \ref{sec: quantile bandits q-sar}.
    Order statistics have been used and studied in various areas, such as robust statistics and extreme value theory.
    The non-asymptotic convergence analysis for order statistics provides a way to understand
    the probability of order statistics deviates from its expectation, and it is useful to support the decision-making with limited samples under uncertainty.

    Let $\{X_t\}_{t=1}^{n}$ be $n$ i.i.d samples drawn from the distribution of $X$,
    and let the $\{X_{(t)}\}_{t=1}^{n}$ be the order statistics of $\{X_t\}_{t=1}^{n}$
    written in decreasing order, i.e.
    \begin{align}
      X_{(1)} \geq X_{(2)} \cdots \geq X_{(n)}.
    \end{align}
    We call $X_{(k)}$ the $k$ rank order statistic, and $X_{(1)}$ and $X_{(n)}$ the maximum
    and minimum respectively.
    Denote the (left-continuous) quantile with $\tau \in (0,1)$ of a random variable $X$ by
    \begin{align}
        \label{equ: quantile definition}
        Q^\tau(X) := \inf \{x: \mathbb{P}(X \leq x) \geq \tau \}.
    \end{align}
    We will refer $Q^\tau(X)$ as $Q^\tau$ whenever $X$ is clear from the context.
    With the empirical c.d.f. defined as $\hat{F}_{n}(x) = \frac{1}{n} \sum_{s = 1}^{n} \mathbb{I}\{X_{s} \leq x\}$,
    the empirical $\tau$-quantile with $n$ samples is defined as
    \begin{align}
    \label{equ: empirical quantile defi}
        \hat{Q}^{\tau}_{n} := \inf \{x : \hat{F}_{n}(x) \geq \tau \}
        = X_{(\lfloor n (1 - \tau) \rfloor)}.
    \end{align}


\subsection{Problem Setting}

We now introduce the family of reward distributions we consider in this work.
    We consider continuous non-negative reward random variables $X$ with p.d.f. $f$ and c.d.f. $F$ which
    satisfy Assumption \ref{ass:IHR}. Note that we are considering distributions
    that are unbounded on the right.


    \begin{defi} [Hazard rate]
        \label{defi: Hazard Rate}
        The hazard rate of a random variable $X$ evaluated at the point $x$ is defined as (assuming density $f(x)$ exists)
         \begin{align}
             h(x)& :=\lim _{\theta \rightarrow 0} \frac{\mathbb{P}\left(x \leq X \leq x + \theta |  X \geq x\right)}{\theta}
             =\frac{f\left(x\right)}{1 - F\left(x\right)}. \nonumber
         \end{align}
    \end{defi}

    \begin{ass}[IHR]
    \label{ass:IHR}
    We consider reward distributions with non-negative support $[0, \infty)$
    having \emph{non-decreasing hazard rate (IHR)},
    i.e. for all $x_1 \geq x_2 \geq 0$, the hazard rate satisfies $h(x_1) \geq h(x_2)$.
    We further suppose that the lower bound of the hazard rate $L := \inf_{x} h(x) > 0$.
    \end{ass}

    The IHR assumption is useful in survival analysis.
    If the hazard rate $h(x)$ increases as $x$ increases, $\mathbb{P}\left(x \leq X \leq x + \theta |  X \geq x\right)$ will increase as well.
    For example, a man is more likely to die within the next month when he is 88 years old than when he is 18 years old.
    Common examples of IHR distributions include the absolute Gaussian, exponential, and Gumbel distributions.
    Log-concave distributions are IHR and have widely been applied to economics, search theory, monopoly theory \citep{bagnoli_log-concave_2005}.

    In the following sections, we show our main results about the concentration bounds for order statistics and quantiles,
    and details are provided in Appendix \ref{sec: concentration proof}.

    \subsection{Order Statistics}
    \label{sec:order statistics}
    Our goal is to derive two exponential rate concentration bounds in terms of rank $k$ order statistics out of $n$ samples. 
    A roadmap of the technical derivations needed is deferred to Section~\ref{sec: key proofs}, and we present only the lemmas 
    needed for analysing BAI in this section.
    For $\gamma \geq 0$,
    the right and left tail are respectively,
    \begin{align}
      \label{equ: right tail bound for order statistics}
      & \mathbb{P}\left(X_{(k)} - \mathbb{E}[X_{(k)}] \geq d^{r}_{k, \gamma}\right) \leq \exp(-\gamma),\\
      \label{equ: left tail bound for order statistics}
      & \mathbb{P}\left(\mathbb{E}[X_{(k)}]-X_{(k)} \geq d^{l}_{k, \gamma}\right) \leq \exp(-\gamma).
    \end{align}
    where $d^{r}_{k, \gamma}, d^{l}_{k, \gamma}$ are the right and left confidence intervals.
    
    To derive such bounds, we consider the entropy method and the Cram\'er-Chernoff method \citep{boucheron2013}.
    These results are used to derive the following lemmas for deviation of order statistics.
    Recall $L$ is the lower bound of the hazard rate, $k$ is chosen from
    the positive integers $\mathbb{N}^\ast$.

    \begin{restatable}[Right Tail Concentration Bounds for Order Statistics]{lemma}{Zk}
      \label{lemma: concentration for Z_k}
      Define $v^r := \frac{2}{k L^2}$, $c^r := \frac{2}{k L}$. 
      Under Assumption \ref{ass:IHR} , for all $\lambda \in [0, 1/c^r)$,
      and all $k \in [1,n) \wedge \mathbb{N}^\ast$, we have
      \begin{align}
      \label{equ: Exponential Efron-Stein inequality Z_k}
      \log \mathbb{E}[\exp\left(\lambda \left(X_{(k)} - \mathbb{E}[X_{(k)}\right)\right)]\leq \frac{\lambda^{2} v^r}{2(1-c^r \lambda)}.
      \end{align}
      For all $\gamma \geq  0$, we obtain the concentration inequality
      \begin{align}
      \label{equ: Concentration inequality for Z_k}
          \mathbb{P}\left( X_{(k)} - \mathbb{E}[X_{(k)}] \geq  \sqrt{2v^r \gamma} + c^r \gamma \right) \leq \exp(-\gamma).
      \end{align}
      \end{restatable}

    \begin{restatable}[Left Tail Concentration Bounds for Order Statistics]{lemma}{NewNegZk}
      \label{lemma: new concentration for -Z_k}
      Define $v^l := \frac{2(n-k+1)}{(k-1)^2 L^2}$. 
      Under Assumption \ref{ass:IHR}, for all $\lambda \geq 0$,
      and all $k \in (1,n] \wedge \mathbb{N}^\ast$, we have
      \begin{align}
      \label{equ: new Exponential Efron-Stein inequality -Z_k}
      \log \mathbb{E}[\exp\left(\lambda \left( \mathbb{E}[X_{(k)}] - X_{(k)}\right)\right)] \leq \frac{\lambda^2 v^l}{2}.
      \end{align}
      For all $\gamma \geq  0$, we obtain the concentration inequality
      \begin{align}
      \label{equ: new Concentration inequality for -Z_k}
          \mathbb{P}\left(  \mathbb{E}[X_{(k)}] - X_{(k)} \geq  \sqrt{2v^l \gamma} \right) \leq \exp(-\gamma).
      \end{align}

    \end{restatable}
    The above results imply $X_{(k)}$ is sub-gamma on the right tail
    with $v^r$ and $c^r$, and sub-Gaussian on the left tail with $v^l$.
    The two different rates of tail bounds reflect the nature of the asymmetric (non-negative) random variable assumption.


    \textbf{Comparison with related work:}
    The result in \citet{boucheron2012} is a special case of Lemma \ref{lemma: concentration for Z_k},
    i.e. when $X_{(k)}$ is the order statistics of absolute Gaussian random variable and $k = n/2$ or 1 only (i.e. median and maximum).
    We extended their results as follows:
    \begin{enumerate*}[mode=unboxed,topsep=0pt,itemsep=-1ex,partopsep=1ex,parsep=1ex,label = \textbf{\textcolor{blue}{(\roman*)}}]
    \item Our results work for a general family of distributions under Assumption \ref{ass:IHR};
    \item We provide a new left tail concentration bound in Lemma \ref{lemma: new concentration for -Z_k};
    \item The concentration result in \citet{boucheron2012} only covered the cases $K = n/2$ or $k=1$,
    while their results can be trivially extended to $k \in [1,n/2]$.
    We show a non-trivial further extension to $k \in [1,n)$ on right tail and $k \in (1,n]$ on left tail.
    \item While we follow similar proof technique (i.e. entropy method) as shown in \citet{boucheron2012},
      we claim novelty of several propositions and lemmas which enables us to derive the new results,
      which can be independent interest, see Remark \ref{remark: Novelty of our proof techinique} in Section \ref{sec: key proofs} for details.
    \end{enumerate*}\\
    \citet{kandasamy_parallelised_ts} extended the result from \citet{boucheron2012} to exponential random variables,
    but we have a tighter left tail bound and a more general analysis in terms of distributions and ranks.
    To our best knowledge, we are the first work studying the two-side order statistic concentration for general IHR distributions.

    \subsection{Quantiles}
    \label{sec: concentration quantiles}

    Now we convert the concentration results for order statistics to quantiles, namely our goal is to derive
    two concentration bounds, for $\gamma \geq 0$,
    \begin{align}
      \label{equ: right tail bound for quantiles}
      & \mathbb{P}\left(\hat{Q}^{\tau}_{n} - Q^{\tau} \geq d^{r}_{n, \tau, \gamma}\right) \leq \exp(-\gamma),\\
      \label{equ: left tail bound for quantiles}
      & \mathbb{P}\left(Q^{\tau}-\hat{Q}^\tau_n \geq d^{l}_{n, \tau, \gamma}\right) \leq \exp(-\gamma).
    \end{align}
    By definition of empirical quantile in Eq. (\ref{equ: empirical quantile defi}), the empirical quantile is the order statistic
    with the rank expressed as a function of quantile level, i.e.
    $\hat{Q}^\tau_{n} = X_{(\lfloor n(1- \tau) \rfloor)}$.
    \citet{david_order_1981} studied the relationship between the expected order statistics and the population quantile under Assumption \ref{ass: pdf continuously differentiable},
    we use their results (Theorem \ref{theo: link expected order statistics and population quantile}) to convert the concentration results of order statistics to quantiles.
    The constant $b$ depends on the density around $\tau$-quantile.
    Linking Theorem \ref{theo: link expected order statistics and population quantile} and Lemma \ref{lemma: concentration for Z_k} or \ref{lemma: new concentration for -Z_k}
    gives the concentration of quantiles (Theorem \ref{theo: New Two-side Concentration inequality for quantiles.}).

    \begin{ass}
      \label{ass: pdf continuously differentiable}
      Assume the probability density function of random variable $X$ is continuously differentiable.
    \end{ass}

    \begin{restatable}[Link expected order statistics and population quantile (\citet{david_order_1981} Section 4.6, \citet{yu2013sample}]{theo}{ESOQ}
      \label{theo: link expected order statistics and population quantile}
      Under Assumption \ref{ass: pdf continuously differentiable},
      there exists constant $b \geq 0$ and scalars $w_n$ such that 
      $w_n = \frac{b}{n}$, then
      $|\mathbb{E}[X_{(\lfloor n(1-\tau) \rfloor)}] - Q^\tau| \leq w_n$.
    \end{restatable}

    \begin{restatable}[Two-side Concentration Inequality for Quantiles]{theo}{TwosideBoundsQuantiles}
        \label{theo: New Two-side Concentration inequality for quantiles.}
        Recall $v^r = \frac{2}{k L^2}$, $v^l = \frac{2(n-k+1)}{(k-1)^2 L^2}$, $c^r = \frac{2}{k L}, w_n = \frac{b}{n}$. 
        For quantile level $\tau \in (0,1)$, let rank $k = \lfloor n(1-\tau) \rfloor$.
        Under Assumption \ref{ass:IHR} and \ref{ass: pdf continuously differentiable},
        we have
        \begin{align}
            \mathbb{P}\left( \hat{Q}^\tau_{n} - Q^\tau \geq \sqrt{2 v^r \gamma} + c^r \gamma + w_n \right) &\leq \exp(-\gamma). \nonumber \\
            \mathbb{P}\left( Q^\tau - \hat{Q}^\tau_{n} \geq \sqrt{2 v^l \gamma} + w_n \right) &\leq \exp(-\gamma). \nonumber
        \end{align}
    \end{restatable}

    Our confidence intervals depend on the number of samples $n$, the quantile level $\tau$ and
    the lower bound of hazard rate $L$.
    Our bound is tighter when $L$ is larger or $\tau$ is smaller.
    Our methods also provide a way to understand the two-sided asymmetric concentration for quantiles when the distributions are asymmetric.


    \textbf{Bias term:} Compared with the concentration results of order statistics (Lemma \ref{lemma: concentration for Z_k} and \ref{lemma: new concentration for -Z_k}),
    there is an extra term $w_n$ in Theorem \ref{theo: New Two-side Concentration inequality for quantiles.}, which has the rate $\mathcal{O}(\frac{1}{n})$ and comes from
    the gap between the expected order statistics and population quantile (Theorem \ref{theo: link expected order statistics and population quantile}).
    Our results show that although the quantile estimations based on single order statistics with finite samples are biased,
    the concentration of
    empirical quantiles to population quantiles has the same convergence rate as the concentration of order statistics to its expectation,
    i.e. both with convergence rate $\mathcal{O}(\frac{1}{\sqrt{n}})$.
    One could potentially consider more than one expected order statistic around the true quantile
    value to obtain a better estimate, which is beyond the scope of this work. But as we will see in Corollary \ref{coro: Rep Con Q}
    the bias term does not affect our error bound for best arms identification.

    \textbf{Comparison to related work:} The main difference of our approach is that
    we directly analyse the object of interest (the random variable itself)
    instead of the the value of its distribution.
    In constrast to proof techniques shown in the literature, our approach is based on the entropy method and does not convert empirical quantiles to empirical c.d.f.
    Instead, we study the concentration bound based on the spacing between consecutive order statistics (See Appendix \ref{sec: key proofs} for details). We provide concentration inequalities to
    two distinct quantities, the expected order statistics and the true quantile.
    Apart from \citet{boucheron2012} we are not aware of any other work on order statistics.

There are two types of concentration inequalities for quantile estimations with the exponential rate in the literature.
Because the empirical quantile is non-linear, the two types of concentration inequalities are not interchangable.
Most of the literature focuses on the concentration of empirical quantile at level $\tau \pm \delta$ ($\forall \delta \geq 0$ s.t. $\tau+\delta \leq 1$ and $\tau - \delta \geq 0$) to the population quantile at level $\tau$, i.e.
\begin{align}
  \mathbb{P}\left(\hat{Q}^{\tau - \delta}_t \geq Q^{\tau}\right) &\leq \exp(-\gamma),\\
  \mathbb{P}\left(\hat{Q}^{\tau + \delta}_t \leq Q^{\tau}\right) &\leq \exp(-\gamma).
\end{align}
Note that (by comparing with Eq. (\ref{equ: right tail bound for quantiles}) and (\ref{equ: left tail bound for quantiles}))
this paper considers a deviation in the quantile, and not $\tau$.
Based on assuming the c.d.f. is continuous and strictly increasing, this type of concentration can benefit from directly converting the concentration to quantiles to the concentration to c.d.f..
For example,
\citet{szorenyi2015qualitative,torossian_x-armed_2019} showed concentration inequalities for quantiles
with rate $\mathcal{O}(\sqrt{\frac{\log n}{n}})$;
\citet{howard_sequential_2019} improved previous work and proved confidence sequences for quantile estimations
with the confidence width shrinks in rate $\mathcal{O}(\sqrt{\frac{\log \log n}{n}})$.

Our results are about a different aspect of deviation, where the concentration is between empirical quantile at level $\tau$ and the population quantile at level $\tau$,
as shown in Eq. (\ref{equ: right tail bound for quantiles}) and (\ref{equ: left tail bound for quantiles}).
\citet{tran-thanh_functional_2014,yu2013sample} proposed concentration for quantile estimations based on the concentration of order statistics
under Chebyshev's inequality, while their concentration inequality is not in exponential rate (in terms of $\gamma$)
thus their bounds decrease much slower when $\gamma$ increases.
A different set of assumptions are needed for this type of concentration to achieve an exponential rate. For example,
\citet{cassel_general_2018} assumed Lipschitz continuity of c.d.f. and derive bounds with rate $\mathcal{O}(\sqrt{\frac{\log n}{n}})$.
By assuming that the c.d.f. is continuous and strictly increasing,
and knowledge of the density around quantiles,
\citet{kolla_concentration_2019} provided an exponential concentration inequality with $\mathcal{O}(\frac{1}{\sqrt{n}})$ by using a generalized notion of an inverse empirical c.d.f. to be able to apply the DKW inequality \citep{dvoretzky1956}.
Their confidence intervals on both two sides are decreasing in rate $\mathcal{O}(\sqrt{\frac{1}{n}})$,
which is comparable to ours.
Our bound can further benefit from the case where quantile level $\tau$ is small or the lower bound of hazard rate $L$ is big.

\section{Roadmap for Concentration Proofs}
    \label{sec: key proofs}

    In this section, we provide the roadmap to the technical results behind Section~\ref{sec:order statistics},
    which may be of independent interest for other applications.
    Figure \ref{fig: Roadmap of concentration proof} summarises the roadmap of the theorems, and
    the detailed proof is shown in Appendix \ref{sec: concentration proof}.
    This section is useful to readers interested in techical aspects of concentration inequalities,
    but may be skipped by others.
    We briefly introduce the \textit{entropy method} here, and
    we refer the reader to \citet{boucheron2013} Chapter 6 for a comprehensive review.
    The logarithmic moment generating function of the random variable $X$ is defined as
    \begin{align}
      \label{equ: logarithmic moment generating function}
      \psi_X(\lambda) := \log \mathbb{E}[\exp(\lambda X)].
    \end{align}
    Define the \textit{entropy} (different from Shannon entropy) of a non-negative random variable $X$ as
    \begin{align}
        \label{equ: entropy}
        \operatorname{Ent}[X] := \mathbb{E}[X \log X]-\mathbb{E} [X] \log \mathbb{E} [X].
    \end{align}
    Then normalising $\operatorname{Ent}[\exp(\lambda X)]$ by $\mathbb{E}[\exp(\lambda X)]$
    gives us an expression in terms of the logarithmic moment generating function
    $\psi_{X - \mathbb{E}[X]}(\lambda)$
    (refer to (\ref{equ: logarithmic moment generating function})), i.e.
    \begin{align}
        \label{equ: normalise entropy}
        \frac{\operatorname{Ent}\left(e^{\lambda X}\right)}{\mathbb{E} e^{\lambda X}}=\lambda \psi_{X - \mathbb{E}[X]}^{\prime}(\lambda)-\psi_{X - \mathbb{E}[X]}(\lambda).
    \end{align}
    One can derive an upper bound for $\operatorname{Ent}[\exp(\lambda X)]$ by the
    modified logarithm Sobolev inequality \citep{ledoux2001concentration}
    (see Theorem \ref{theo: Modified logarithmic Sobolev inequality} in Appendix \ref{sec: concentration proof}).
    Then by solving a differential inequality, tail bounds can be obtained via Chernoff's bound.

    \begin{figure}[!t]
        \centering
        \includegraphics[scale=0.25]{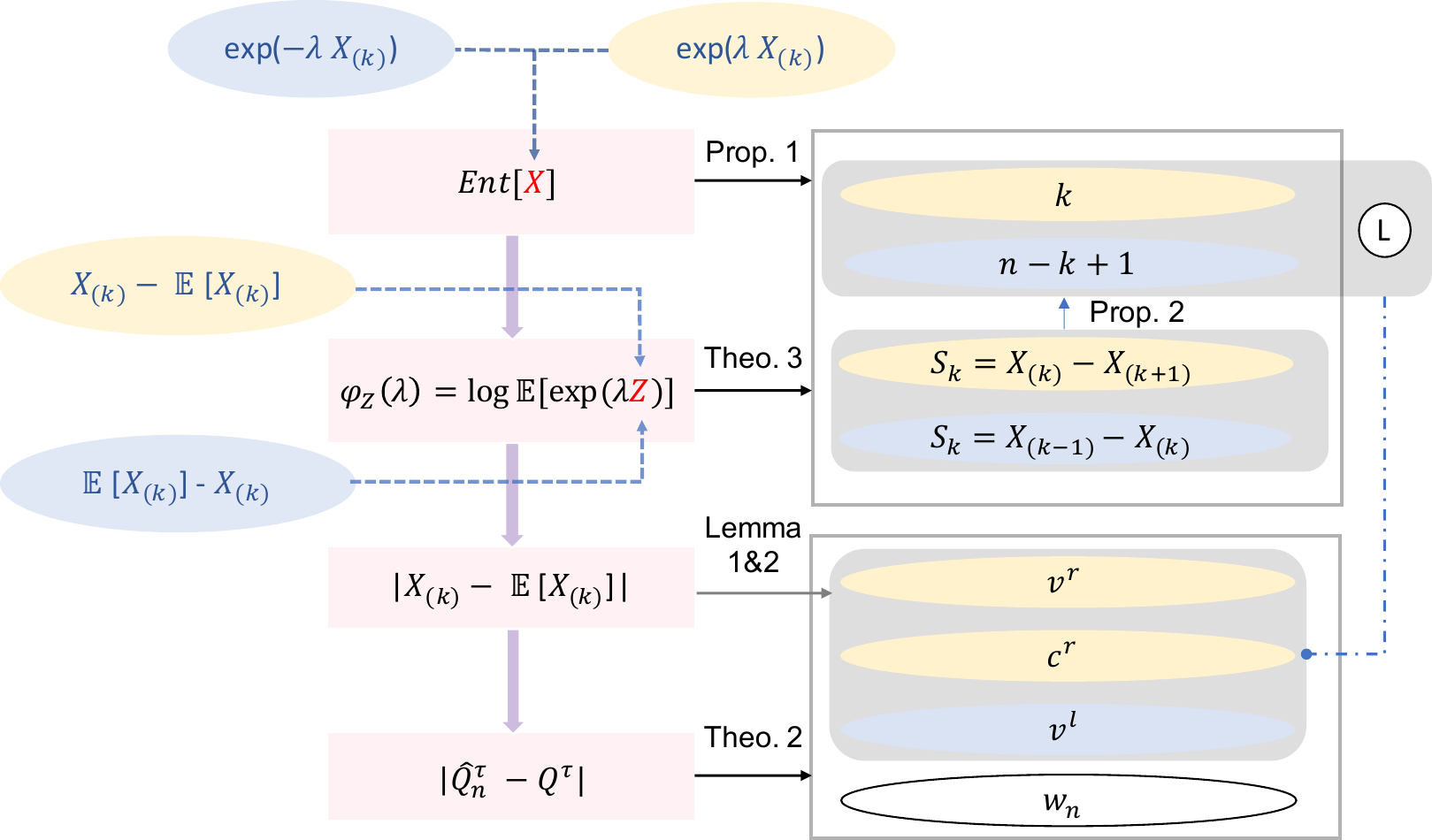}
        \caption{Roadmap of concentration proof. 
        Variables related to left tail bound and right tail bound are specified by blue and yellow respectively.
        We upper bound the variables highlighted in pink in the order indicated by arrows between them. 
        The upper bounds are functions of the blocks of variables pointed by solid arrows, with the corresponding theorem names on the arrow. 
        }
        \label{fig: Roadmap of concentration proof}
        \vspace{-0.5cm}
    \end{figure}

    In the following, we apply the entropy method to order statistics and focus on our contributions in terms of the proof technique.
    We derive a technical result to bound the entropy in Proposition~\ref{prop: new bounds for entropy}.
    The bound on entropy, along with another technical result on the spacing between consecutive order statistics allows us to derive 
    an exponential Efron-Stein inequality (Theorem~\ref{theo: New Extended Exponential Efron-Stein inequality}).
    Recall $X_{(1)} \geq \ldots \geq X_{(n)}$ are the order statistics of $X_1, \ldots, X_n$.
    Define the spacing between order statistics of order $k$ and $k-1$ as
    \begin{align}
    \label{equ: defi spacing}
      S_k := X_{(k)} - X_{(k+1)}; \quad S_{k-1} := X_{(k-1)} - X_{(k)}.
    \end{align}
    We first show the upper bounds of entropy in terms of the spacing between order statistics in Proposition \ref{prop: new bounds for entropy}.

    \begin{restatable}[Entropy upper bounds]{prop}{OSS}
    \label{prop: new bounds for entropy}
    Define $\phi(x) :=\exp(x)-x-1$ and $\zeta(x) :=\exp(x) \phi(-x) =1+(x-1) \exp(x)$.
    For all $\lambda \geq 0$,
    and for $k \in [1,n) \wedge \mathbb{N}^\ast$,
    \begin{align}
    \label{equ: new bounds for entropy X}
        \operatorname{Ent}\left[\exp(\lambda X_{(k)})\right] \leq k \mathbb{E}\left[\exp(\lambda X_{(k+1)}) \zeta\left(\lambda S_k \right)\right].
    \end{align}
    For $k \in (1,n] \wedge \mathbb{N}^\ast$,
    \begin{align}
    \label{equ: new bounds for entropy -X}
    \operatorname{Ent} & \left[\exp(- \lambda X_{(k)})\right] \leq \nonumber \\
         & (n - k + 1)
         \mathbb{E}\left[\exp(-\lambda X_{(k)}) \phi\left( - \lambda S_{k-1}  \right)\right].
    \end{align}

    \end{restatable}



    The upper bounds in Proposition \ref{prop: new bounds for entropy}
    are expressed in terms of the corresponding order statistics and the spacing for consecutive order statistics.
    From Proposition \ref{prop: new bounds for entropy}, and by normalising entropy as shown in Eq. (\ref{equ: normalise entropy}),
    we show the upper bound of the logarithmic moment generating function of $Z_k := X_{(k)} - \mathbb{E}[X_{(k)}]$ and $Z_k^\prime :=  \mathbb{E}[X_{(k)}] - X_{(k)}$.
    \begin{restatable}[Extended Exponential Efron-Stein inequality]{theo}{NewExtendedEES}
        \label{theo: New Extended Exponential Efron-Stein inequality}
        With the logarithmic moment generating function defined in Eq. \ref{equ: logarithmic moment generating function},
        for $\lambda \geq 0$ and  $k \in [1,n) \wedge \mathbb{N}^\ast$,
        \begin{align}
            \label{equ: new Exponential Efron-Stein inequality right}
            & \psi_{Z_k}(\lambda)  \leq \lambda \frac{k}{2} \mathbb{E} \left[S_{k}\left(\exp(\lambda S_{k})-1\right)\right].
        \end{align}
        For $k \in (1,n] \wedge \mathbb{N}^\ast$,
        \begin{align}
            \label{equ: new Exponential Efron-Stein inequality left}
            & \psi_{Z_k^\prime}(\lambda) \leq \frac{\lambda^2 (n-k+1)}{2} \mathbb{E}[S_{k-1}^2].
        \end{align}
    \end{restatable}
    Observe that the upper bounds in Theorem \ref{theo: New Extended Exponential Efron-Stein inequality}
    depend on the order statistics spacings $S_k, S_{k-1}$ in expectation. 

    The non-decreasing hazard rate assumption (Assumption~\ref{ass:IHR}) allows us to upper bound the spacings in expectation.
    We show the upper bound of expected spacing in Proposition \ref{prop: bound of expected spacing}.
    Based on a similar proof technique of Proposition \ref{prop: bound of expected spacing},
    we can further bound Theorem \ref{theo: New Extended Exponential Efron-Stein inequality} and the results are shown in
    Lemma \ref{lemma: concentration for Z_k} and Lemma \ref{lemma: new concentration for -Z_k}.

    \begin{restatable}{prop}{BoundExpSpacing}
    \label{prop: bound of expected spacing}
    For any $k \in [1, n) \wedge \mathbb{N}^\ast$, the expectation of spacing $S_k$ defined in Eq. (\ref{equ: defi spacing}) can be bounded under Assumption \ref{ass:IHR},
    $\mathbb{E}[S_k] \leq \frac{1}{kL}.$
    \end{restatable}



    \begin{remark}[Novelty of our proof techinique]
    \label{remark: Novelty of our proof techinique}
    The results shown in this section may be of independent interest.
    \begin{enumerate*}[mode=unboxed,topsep=0pt,itemsep=-1ex,partopsep=1ex,parsep=1ex,label = \textbf{\textcolor{blue}{(\roman*)}}]
    \item In Proposition \ref{prop: new bounds for entropy}, we show the upper bounds of both $\operatorname{Ent}\left[\exp(\lambda X_{(k)})\right]$
          and $\operatorname{Ent}\left[\exp(-\lambda X_{(k)})\right]$ for all rank $k$ except extremes. This allows two-sided tail bounds to hold for all ranks except extremes.
    \item In Theorem \ref{theo: New Extended Exponential Efron-Stein inequality}, we show upper bounds of logarithmic moment generating function for both
          $X_{(k)} - \mathbb{E}[X_{(k)}]$ and $\mathbb{E}[X_{(k)}] - X_{(k)}$ w.r.t the order statistics spacing in expectation.
          The upper bound of $\mathbb{E}[X_{(k)}] - X_{(k)}$ is tighter (sub-Gaussian).
    \item We propose an upper bound for the expected order statistics spacing $S_k = X_{(k)} - X_{(k+1)}$ in Proposition \ref{prop: bound of expected spacing}.
    \end{enumerate*}
    \end{remark}

\section{Quantile Bandits Policy: Q-SAR}
    \label{sec: quantile bandits q-sar}

    
    
    We consider the setting of multi-armed bandits with a finite number of arms $K$.
    For each arm $i \in \mathcal{K} = \{1, ..., K\}$,
    the rewards are sampled from an unknown stationary reward distribution $F_i$.
    We assume arms are independent.
    The environment $\nu$ consists of the set of all reward distributions $F_i$, i.e. $\nu := \{F_i: i \in \mathcal{K}\}$.
    The agent makes a sequence of decisions based on a policy $\pi$
    for $N$ rounds, where each round is denoted by $t\in \{1,\dots, N\}$.
    We denote the arm chosen at round $t$ as $A_t$,
    and $T_i(t)$ as the number of times for arm $i$ was chosen at the end of round $t$, i.e.
    $
        T_i(t) := \sum_{s = 1}^t \mathbb{I}\{A_s = i\}.
    $
    At round $t$, the agent observes reward $X^{A_t}_{T_{A_t}(t)}$ sampled from $F_{A_t}$.
    
    The quality of an arm is determined by the $\tau$-quantile of its reward distribution.
    Arms with higher $\tau$-quantile values are better.
    We order the arms according to optimality as $o_1, \dots, o_K$ s.t. $Q^\tau_{o_1} \geq \dots \geq Q^\tau_{o_K}$.
    The optimal arm set of size $m$ is $\mathcal{S}_m^\ast = \{o_1, \dots, o_m\}$.
    Without loss of generality, we assume $ \mathcal{S}_m^\ast$ is unique.
    Following \citet{audibert2010best,bubeck2013multiple}, we formulate our objective by the probability of error.
    
    \begin{defi}[Probability of error/misidentification]
    \label{defi: probability of error}
    We denote $\mathcal{S}_m^N \subset \mathcal{K}$ as the set of $m$ arms returned by
    the policy at the end of the exploration phase. 
    Define the probability of error as
    \begin{align}
      e_N := \mathbb{P}\left(\mathcal{S}_m^N \neq \mathcal{S}_m^\ast\right).
    \end{align}
    \end{defi}
    
    The goal is, with a fixed budget of $N$ rounds, to design a policy which returns a set of arms of size $m \geq 1$
    so that probability of error $e_N$ is minimised.

    \begin{algorithm}[t]
      \caption{Q-SAR}
      \label{alg:Q-SAR}
      \begin{algorithmic}
          \STATE Let the active set $\mathcal{A}_1= \{1, ..., K\}$, the accepted set $\mathcal{M}_1 = \emptyset$, the number of arms left to find $l_1 = m$,
          \STATE $\overline{\log} (K) = \frac{1}{2} + \sum_{i = 2}^K \frac{1}{i}, n_0 = 0$, and for $p \in \{1, ..., K-1\},$
          $n_p = \left\lceil\frac{1}{\overline{\log }(K)} \frac{N-K}{K+1-p}\right\rceil.$
          \FOR{each phase $p = 1,2, ..., K-1$}
          \STATE (1) For each $i \in \mathcal{A}_p$, sample arm $i$ for $n_p - n_{p-1}$ times.
          \STATE (2) For each $i \in \mathcal{A}_p$, sort the empirical quantile of arm $i$ in a non-increasing order, denote the sorted arms as $a_{best}, a_2, ..., a_{l_p}, a_{l_p+1}, ..., a_{worst}$ (with ties broken arbitrarily).
          \STATE (3) Calculate $\widehat{\Delta}_{best} = \hat{Q}^\tau_{a_{best}, n_p} - \hat{Q}^\tau_{a_{l_p + 1}, n_p}$,
          $\widehat{\Delta}_{worst} =  \hat{Q}^\tau_{a_{l_p}, n_p} - \hat{Q}^\tau_{a_{worst}, n_p}$.
          \IF {$\widehat{\Delta}_{best} > \widehat{\Delta}_{worst}$}
          \STATE $\mathcal{A}_{p+1} = \mathcal{A}_{p} \backslash \{a_{best}\}$,
           $\mathcal{M}_{p + 1} = \mathcal{M}_{p} \cup \{a_{best}\}$,
           \STATE $l_{p+1} = l_p -1$.
          \ELSE
           \STATE $\mathcal{A}_{p+1} = \mathcal{A}_{p} \backslash \{a_{worst}\}$,
           $\mathcal{M}_{p + 1} = \mathcal{M}_{p},$
           \STATE $l_{p+1} = l_p$.
          \ENDIF
          \ENDFOR
          \STATE Return the $m$ arms $\mathcal{S}_m^N = \mathcal{A}_{K} \cup \mathcal{M}_{K}$.
      \end{algorithmic}
    \end{algorithm}


    We propose a Quantile-based Successive Accepts and Rejects (Q-SAR) algorithm (Algorithm \ref{alg:Q-SAR}), which adapts the \textit{Successive Accepts and Rejects} (SAR) algorithm \citep{bubeck2013multiple}.
    We divide the budget $N$ into $K-1$ phases.
    The number of samples drawn for each arm in each phase remains the same as in the \citet{bubeck2013multiple}.
    At each phase $p \in \{1, 2, \dots, K-1\}$, we maintain two sets (refer to Figure~\ref{fig: QSAR illustration.}):
    i) the active set $\mathcal{A}_p$, which contains all arms that are actively drawn in phase $p$;
    ii) the accepted set $\mathcal{M}_p$, which contains arms that have been accepted.
    In each phase $p$, an arm is removed from the active set, and it is either accepted or rejected.
    The accepted arm is added into the accepted set.
    At the end of phase $K-1$, only one arm remains in the active set.
    The last remaining arm together with and the accepted set (containing $m-1$ arms) form the returned recommendation $\mathcal{S}_m^N$.

    \begin{figure}[t]
      \centering
      \includegraphics[scale=0.22]{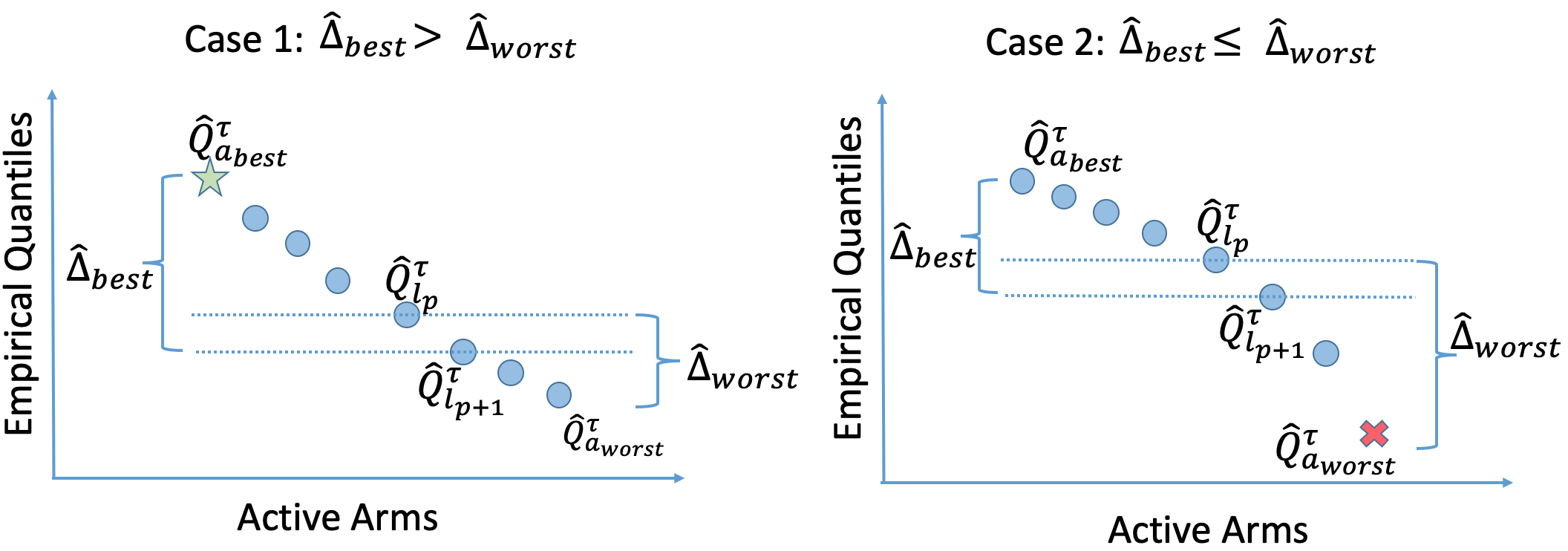}
      \caption{Phase $p$ of Q-SAR illustration. Star indicates accepted arm, cross indicates rejected arm.
        }
        \label{fig: QSAR illustration.}
    \end{figure}

    We can consider the task of identifying $m$ best arms as grouping arms into the optimal set $\mathcal{S}_m^\ast$ and non-optimal set.
    Then intuitively it is easier to firstly group arms which are farther away from the boundary of the two groups,
    since the estimation error is less likely to influence the grouping (Refer to Figure \ref{fig: QSAR illustration.}).
    Q-SAR follows this intuition and determines to accept or reject the arm which is the farthest (based on estimates) from the boundary in each phase.

    We introduce a simplified version of the SAR algorithm.
    Instead of considering all empirical gaps as SAR algorithm \citep{bubeck2013multiple} proposed,
    Q-SAR decides whether to accept or reject an arm by only comparing two empirical gaps:
    $\widehat{\Delta}_{best}$ and $\widehat{\Delta}_{worst}$ (defined in Algorithm \ref{alg:Q-SAR} step (3)).
    This simplification also applies when using the mean as a summary statistic, and results in an equivalent algorithm to the original SAR in \citep{bubeck2013multiple}.
    If $\widehat{\Delta}_{best} > \widehat{\Delta}_{worst}$, the arm with maximum empirical quantile is accepted;
    otherwise, the arm with minimum empirical quantile is rejected.

    When $m=1$, $\widehat{\Delta}_{worst} = \hat{Q}^\tau_{a_{best}} - \hat{Q}^\tau_{a_{worst}}$,
    and $\widehat{\Delta}_{best} = \hat{Q}^\tau_{a_{best}} - \hat{Q}^\tau_{a_{2}}$. For all phases,
    $\widehat{\Delta}_{worst} \geq \widehat{\Delta}_{best}$, thus Q-SAR will keep rejecting arms.
    In this case, Q-SAR is the same as Q-SR (Algorithm \ref{alg:Q-SR}, shown in Appendix \ref{sec: experiment details}). 

    \subsection{Theoretical Analysis}

    Recall optimality of arms are denoted as
    $o_1, \dots, o_K$ s.t. $Q^\tau_{o_1} \geq \dots \geq Q^\tau_{o_K}$.
    The optimal arm set of size $m$ is $\mathcal{S}_m^\ast = \{o_1, \dots, o_m\}$.
    For each arm $i \in \mathcal{K}$, we define the \textit{gap} $\Delta_i \geq 0$ by
    $$
    \Delta_{i} 
    := \left\{\begin{array}{ll}
      Q^\tau_{i}-Q^\tau_{o_{m+1}} & \text { if } i \in  \mathcal{S}_m^\ast; \\
    Q^\tau_{o_{m}}- Q^\tau_{i} & \text { if } i \notin  \mathcal{S}_m^\ast.
    \end{array}\right.
    $$
    We sort the gaps in a non-decreasing order and denote the $i^{th}$ gap as $\Delta_{(i)}$,
    i.e. $\Delta_{(1)} \leq \Delta_{(2)} \leq \dots \leq \Delta_{(K)}$.
    The gaps characterise how separate the arms are and reflect the hardness of the problem.
    The smaller the (minimum) gaps of the arms are, the harder the BAI task is.
    We define the \textit{problem complexity} $H^\tau$ as
    \begin{align}
      \label{equ: problem complexity H}
      H^\tau = \max_{\{i, j \in \mathcal{K}\}} \frac{8 j}{1-\tau}(\frac{4\alpha}{L_i^2  \Delta_{(j)}^2} + \frac{\beta_i}{L_i^2  \Delta_{(j)}}).
    \end{align}
    where $\alpha = \frac{4(1+\tau)}{1-\tau}, \beta_i = \frac{4}{3}(2 L_i + b_i(1-\tau)L_i^2)$,  with $L_i$ as the lower bound of hazard rate of arm $i$.



    To bound the probability of error under Q-SAR policy, it is convenient to re-express
    Theorem~\ref{theo: New Two-side Concentration inequality for quantiles.} such that
    the deviation between the empirical and true quantile is given by $\epsilon$.
    Observe that for Q-SAR, we are interested in events of small probability, that is
    for large values of $\gamma$ in Theorem~\ref{theo: New Two-side Concentration inequality for quantiles.}.
    In the corollary below, we focus on such events of small probability by considering
    $\gamma \geq 1$ (i.e. error less than $\frac{1}{e}\approx 0.37$), which allows
    a simpler expression.

    \begin{restatable}[Representation of Concentration inequalities for Quantiles]{coro}{RepConQ}
      \label{coro: Rep Con Q}
      For $\epsilon > 0$, $v^r, v^l, c^r, w_n$ stay the same as stated in Theorem \ref{theo: New Two-side Concentration inequality for quantiles.}.
      With $\gamma \geq 1$,
      Theorem \ref{theo: New Two-side Concentration inequality for quantiles.}
      can be represented as
      \begin{align}
          & \mathbb{P}\left(\hat{Q}^\tau_{n} - Q^\tau \geq \epsilon \right) \leq \exp\left(-\frac{{\epsilon }^2}{2 (v^r+ (c^r + w_n) \epsilon )}\right), \nonumber\\
          & \mathbb{P}\left( Q^\tau  - \hat{Q}^\tau_{n} \geq \epsilon \right) \leq \exp\left(- \frac{{\epsilon } ^ 2}{2 (v^l + w_n \epsilon)}\right). \nonumber
      \end{align}
    \end{restatable}

    Applying Corollary \ref{coro: Rep Con Q},
    we show an upper bound of the probability of error in Theorem \ref{theo: QSAR}.
    Note we assume the total budget is at least $\frac{4}{1-\tau} \overline{\log}(K) + K$,
    which guarantees that after the initial round each arm has enough samples (See Lemma \ref{lemma: quantile concentration, dependency on n} in Appendix for details).
    We also present an error bound without assuming the lower bound of the budget in Theorem \ref{theo: QSAR variant, with constant term}, based on concentration results in Lemma \ref{lemma: quantile concentration, dependency on n, no lower bound on n assumption}.

    \begin{restatable}[Q-SAR Probability of Error Upper Bound]{theo}{QSAR}
      \label{theo: QSAR}
      For the problem of identifying $m$ best arms out of $K$ arms, with budget $N \geq \frac{4}{1-\tau} \overline{\log}(K) + K$, the probability of error (Definition \ref{defi: probability of error}) for Q-SAR satisfies
      \begin{align}
        e_N \leq 2 K^2 \exp\left(- \frac{N-K}{\overline{\log}(K) H^\tau}\right), \nonumber
      \end{align}
      where problem complexity $H^\tau$ is defined in Eq. (\ref{equ: problem complexity H}).
      \end{restatable}


    Observe the error bound depends on the problem complexity and has rate $\mathcal{O}(K^2 \exp{(\frac{-N + K}{\log K})})$ w.r.t the number of arms $K$ and budget $N$.
    The smaller $H^\tau$ is, the smaller the upper bound of error probability is.
    In the following, we show a sketch of the proof.
    The detailed proof is provided in Appendix \ref{sec: Q-SAR proof}.

    \begin{proof-sketch}
    Define the event $\xi$,
    \begin{align}
    \label{equ: event xi}
      \xi :=& \{\forall i \in\{1, \ldots, K\}, p \in\{1, \ldots, K-1\}, \nonumber\\
      & \left| \hat{Q}^\tau_{i, n_p} - Q^\tau_i \right| < \frac{1}{4} \Delta_{(K+1-p)}\},
    \end{align}
    where $n_p$ is the number of samples at phase $p$ for arm $i$.
    One can upper bound $\mathbb{P}(\bar{\xi})$, i.e. the probability that complementary event of $\xi$ happens,
    by the union bound and our concentration results (Corollary \ref{coro: Rep Con Q}).
    Then it suffices to show Q-SAR does not make any error on event $\xi$, 
    which implies
    (I) no arm from the optimal set is rejected and (II) no arm from non-optimal set is accepted.

    To show Q-SAR does not make any error on event $\xi$, 
    we prove by induction on $p \geq 1$.
    In phase $p$, there are $K + 1 - p$ arms inside of the active set $\mathcal{A}_p$,
    we sort the arms inside of $\mathcal{A}_p$ and denote them as $\ell_1, \ell_2, \dots, \ell_{K+1-p}$
    such that $Q^\tau_{\ell_1} \geq Q^\tau_{\ell_2} \geq \dots \geq Q^\tau_{\ell_{K+1-p}}$.
    We assume there is no wrong decision made in all previous $p-1$ phases and prove by contradiction.
    We assume that one arm from non-optimal set is accepted,
    or one arm from the optimal set is rejected in phase $p$.
    These two assumptions give us
    \begin{align}
      \Delta_{(K+1-p)} > \max \left\{Q^\tau_{\ell_1} - Q^\tau_{o_{m+1}}, Q^\tau_{o_m} - Q^\tau_{\ell_{K+1-p}}\right\}, \nonumber
    \end{align}
    which contradicts with the fact that $\Delta_{(K+1-p)} \leq \max \{Q^\tau_{\ell_1} - Q^\tau_{o_{m+1}}, Q^\tau_{o_{m}} - Q^\tau_{\ell_{K+1-p}}\}$,
    since there are only $p-1$ arms that have been accepted or rejected at phase $p$.
    This concludes the proof.
    \end{proof-sketch}


    \textbf{Comparison to related work:}
    There are two problems in the standard MAB, namely the \textit{regret minimisation} problems \cite{Auer2002} and the \textit{Best Arm Identification} (BAI) problems \cite{audibert2010best}.
    The goal of regret minimisation problems in bandit setting \cite{Auer2002} is to maximise the cumulative rewards, i.e. minimise the cumulative regret. 
    Best arm identification has been studied for fixed budget \citep{audibert2010best}
    and fixed confidence \citep{even2006action} settings.
    The difference between the two settings is how the exploration phase is terminated (when the budget runs out
    or when the quality of recommendations is at a fixed confidence level).
    We focus on the fixed budget setting for this work.
    For a comprehensive review of bandits we refer to \citet{lattimore2018bandit}.\\
    Previous quantile related BAI work \cite{szorenyi2015qualitative,david_pure_2016,yu2013sample,howard_sequential_2019}  mostly focused on another setting of BAI, the fixed confidence setting.
    Literature concerning quantile bandits with the fixed budget is scarce.
    The most related work is \citet{tran-thanh_functional_2014}, which studied functional bandits, with quantiles as one example.
    They proposed the quantile based batch elimination algorithm.
    Since their concentration inequality is based on Chebyshev's inequality (and hence not exponential),
    the upper bound on the probability of error has a rate $\mathcal{O}(K^2/N)$, which is slower than ours.
    \citet{torossian_x-armed_2019} considered the fixed budget setting but focused on quantile optimization on stochastic black-box functions,
    which is different from our setting.
    As far as we know, Q-SAR is the first policy designed to identify multiple arms with highest $\tau$-quantile values.\\
    The upper bound of error probability of Q-SAR and SAR~\citep{bubeck2013multiple} have the same rate
    (in terms of the budget $N$ and number of arms $K$) up to constant factors.
    Our constant term is smaller when the minimum lower bound of hazard rate takes a larger value.
    Unlike the mean-based algorithm, $H^\tau$ depends on the quantile level $\tau$ as well.
    The smaller $\tau$ is, the smaller the $H^\tau$ is. This can be intuitively explained as needing more samples to estimate higher level quantiles for IHR distributions. \\



\section{Experiments}
\label{sec: experiments}

In this section, we illustrate how the proposed Q-SAR algorithm works on a toy example
(Section~\ref{sec:toy-example}) and demonstrate the empirical performance
on a vaccine simulation (Section~\ref{sec:vaccine-simulation})\footnote{\small \url{https://github.com/Mengyanz/QSAR}}.

\subsection{Illustrative Example}
\label{sec:toy-example}
We set up simulated environments by constructing three arms with absolute Gaussian distribution or exponential distribution.
The summary statistics of reward distributions are shown in Table \ref{table: summary statistics for simulated rewards}.
We expand the size of environments by replicating arms.
Details about the experimental setting are shown in Appendix \ref{sec: experiment details}.


We design two environments (sets of reward distributions) to show how one can benefit from considering quantiles
as summary statistics with our algorithm.
The design of the two environments reflects the two motivations in Section \ref{sec: Introduction}.
Let $K$ be the total number of arms and $m$ be the number of arms to recommend.
For each environment, we choose $m=1$ (single best arm identification) and $m = 5$.
We evaluate the probability of errors defined in Eq. (\ref{defi: probability of error}).
As a comparison, we introduce a Quantile-based Successive Rejects (Q-SR) algorithm in Algorithm \ref{alg:Q-SR} (Appendix \ref{sec: Q-SR appendix}),
which is adapted from Successive Rejects algorithm \citep{audibert2010best} and we modify it to recommend multiple arms.

\begin{table}
  \centering
  \resizebox{0.4\textwidth}{!}{
  \begin{tabular}{|c|c|c|c|}
    & Mean   & 0.5-Quantile & 0.8-Quantile\\
  A     & 1.60   & 1.35 & 2.55\\
  B     & 3.60   & \bf{3.50}  & 5.21\\
  C     & \bf{4.00}  & 2.76  &  \bf{6.42}\\
  Opt Arm & C      & B   &  C\\
  Min Gap & 0.40     & 0.74   &  1.21
  \end{tabular}
  }
  \caption{Summary Statistics of Reward Distributions}
  \label{table: summary statistics for simulated rewards}
  \vspace{-0.5cm}
\end{table}



\begin{figure}[t!]
    \centering
    \includegraphics[scale=0.34]{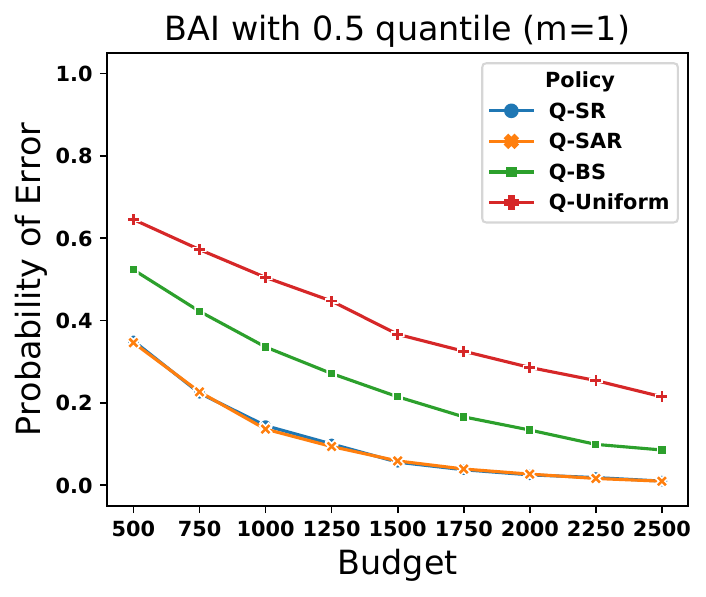}
    \includegraphics[scale=0.34]{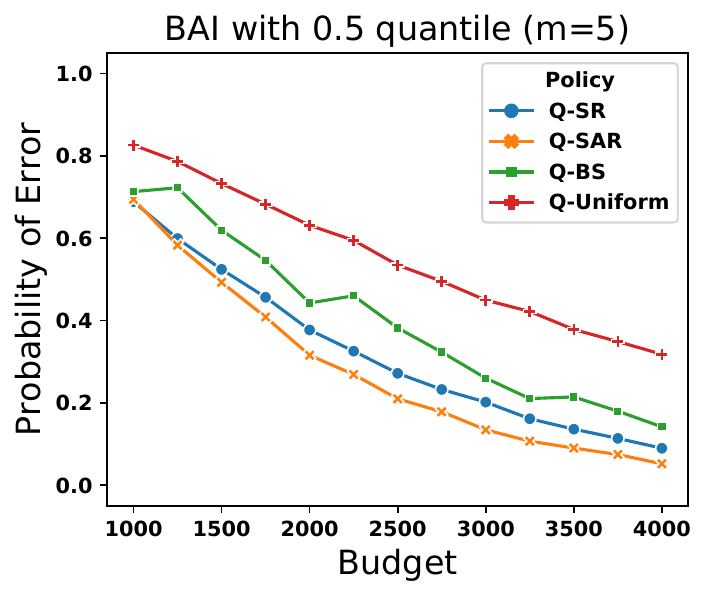}
    \includegraphics[scale=0.34]{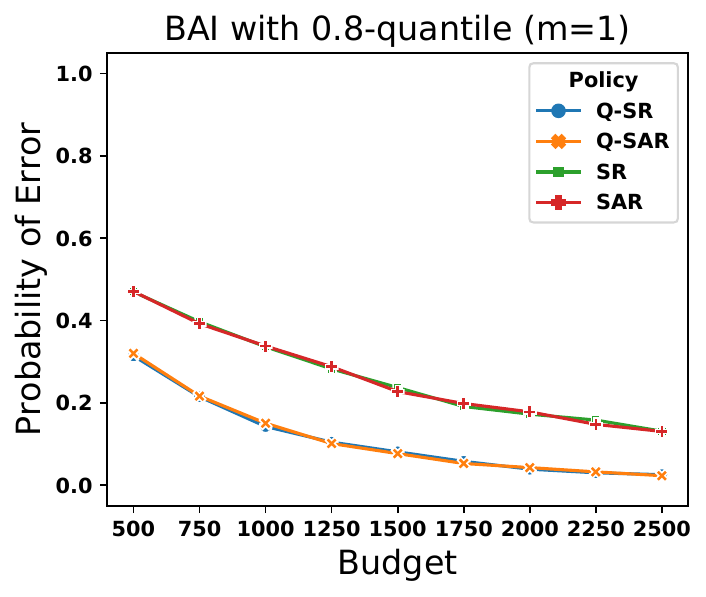}
    \includegraphics[scale=0.34]{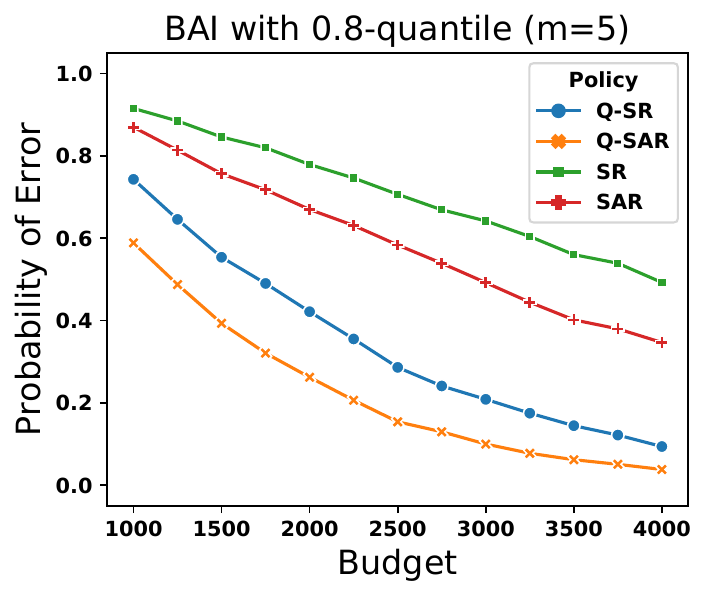}
    \caption{Illustrative examples.
    The first row shows the Environment I simulation where arms are evaluated by 0.5-quantiles.
    The second row shows the Environment II simulation where arms are evaluated by 0.8-quantiles (mean provides the same order of arms).
    We consider the task of recommending a single arm (left column) 
    and recommending multiple arms ($m=5$, right column). 
    The performance is evaluated in terms of the probability of error (with 5000 independent runs).
    }
    \label{fig: Simulated experiment}
    \vspace{-0.5cm}
  \end{figure}
  \begin{figure*}[t!]
    \centering
    \includegraphics[scale=0.35]{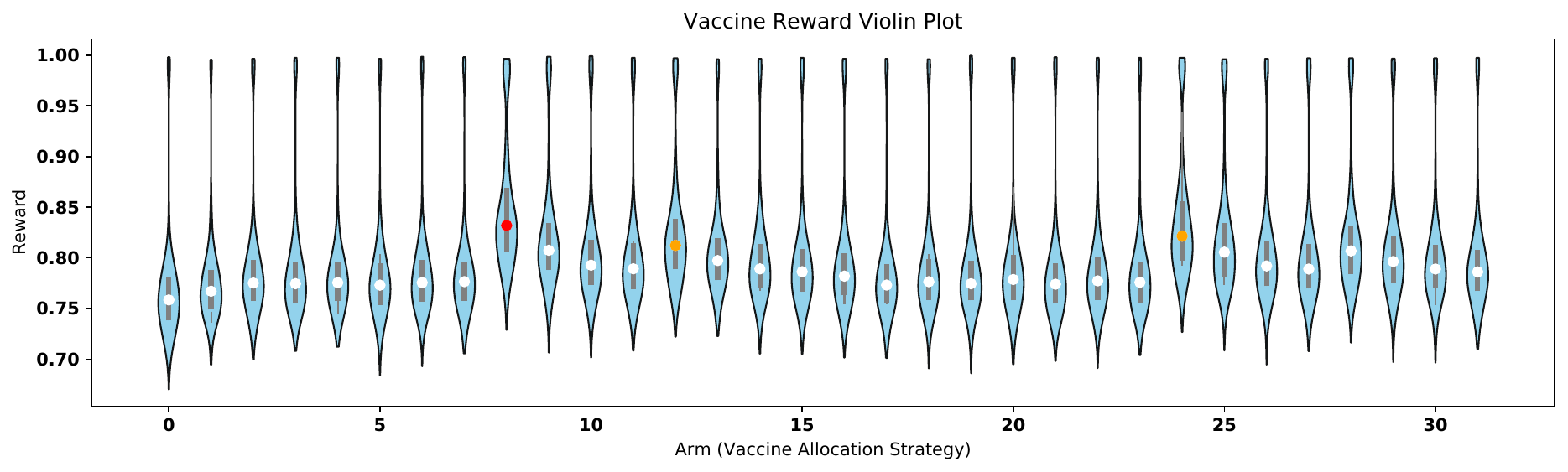}
    \caption{Vaccine Reward Violin Plot. The X-axis represents arms (vaccine allocation strategies), Y-axis represents rewards,
    which is the proportion of individuals that did not experience symptomatic infection. The circle in each violin
    represents the median, where the red one is the highest and the orange ones are the second and the third highest.
    The black line in each violin shows the range of 0.25-quantile to 0.75-quantile.
    }
    \label{fig: Vaccine Reward Violin Plot.}
  \end{figure*}
\textbf{Environment I:}
We consider $K = 20 + m$ arms with 15 A arms, $m$ B (optimal) arms, and 5 C arms.
The goal is to identify $m$ arms with largest $0.5$-quantile (i.e. median).
We compare our algorithms with the quantile-based baseline algorithms:
\begin{enumerate*}[mode=unboxed,topsep=0pt,itemsep=-1ex,partopsep=1ex,parsep=1ex,label = \textbf{\textcolor{blue}{(\roman*)}}]
\item Quantile uniform sampling (Q-Uniform), where each arm is sampled uniformly and we select the arm with the maximum 0.5-quantile;
\item Quantile Batch Elimination (Q-BE) proposed in \citet{tran-thanh_functional_2014} 
\item Quantile-based Successive Rejects (Q-SR). 
\end{enumerate*}

\textbf{Environment II:}
We consider $K = 20 + m$ arms with 15 A arms, 5 B arms, and $m$ C (optimal) arms.
The goal is to identify $m$ arms with the maximum $0.8$-quantile.
Both mean and 0.8-quantile provide the same order of arms, while
0.8-quantile can provide a larger gap compared with the mean.
According to Theorem \ref{theo: QSAR}, the environment with a larger minimum gap has a smaller probability complexity
and thus smaller upper bound of the probability of error (it holds for the mean-based algorithm \citep{bubeck2013multiple} as well).
We compare our algorithms with the baseline algorithms:
\begin{enumerate*}[mode=unboxed,topsep=0pt,itemsep=-1ex,partopsep=1ex,parsep=1ex,label = \textbf{\textcolor{blue}{(\roman*)}}]
\item mean-based Successive Accepts and Rejects (SAR) \citep{bubeck2013multiple}.
\item mean-based Successive Rejects (SR) \citep{audibert2010best}.
\item Quantile-based Successive Rejects (Q-SR). 
\end{enumerate*}

\textbf{Results:} We show the empirical probability of error as a function of budget
in Figure \ref{fig: Simulated experiment}.
Q-SAR has the best performance under all settings.
Q-SAR and Q-SR has the same performance for the single best arm identification ($m=1$) task in both environments,
while Q-SAR outperforms Q-SR for multiple identifications ($m=5$).
For Environment II,
Q-SAR and Q-SR outperform SAR and SR since the gap between the optimal and suboptimal arms
is bigger when evaluating arms by 0.8-quantiles than by means,
and Q-SAR has a clear lower probability of error than Q-SR when the sample size is small.




\subsection{Vaccine Simulation}
\label{sec:vaccine-simulation}

  We consider the problem of identifying optimal strategies for allocating an influenza vaccine.
  Following~\citet{LibinBAI:2017}, we format this problem as an instance of the BAI where each vaccine allocation strategy is an arm.
  Details of allocation strategies are available in the Appendix \ref{sec: vaccine simulation appendix}.
  The reward of a strategy is defined as the proportion of individuals that did not experience symptomatic infection.
  We generate 1000 rewards for each strategy by simulating the epidemic for 180 days using FluTE~\footnote{\url{https://github.com/dlchao/FluTE}}
  (with basic reproduction number $R_0 = 1.3$).
  The violin plot of reward samples is shown in Figure~\ref{fig: Vaccine Reward Violin Plot.}.
  The empirical reward distribution of arms are IHR, but with outliers close to 1 (which violate IHR).
  These outliers are due to the fact that the pathogen does not result in an epidemic
  in some simulation runs, which does not reflect the efficacy of the vaccine.

  We use the median ($\tau=0.5$) as a robust summary statistic for each strategy.
  We apply our Q-SAR algorithm on two tasks. 
  (1) the task of identifying $m=1$ best arm (index 8) with a fixed budget ranging from 500 to 2,500, and
  (2) the task of identifying $m=3$ best arms (index 8, 24 and 12) with a fixed budget ranging from 1,000 to 4,000.
  The performance for the BAI task is shown in Figure \ref{fig: Vaccine Simulated experiment}.
  We compare our algorithm other quantile-based algorithms only,
  since $0.5-$quantile and means results in different optimal arms.
  The empirical evidence shows Q-SAR is the best for multiple best arms identification
  (when $m=3$) and is robust to outliers.
  We leave the theoretical analysis for the outlier robustness of our approach as future work.

  \begin{figure}[t]
    \centering
    \includegraphics[scale=0.3]{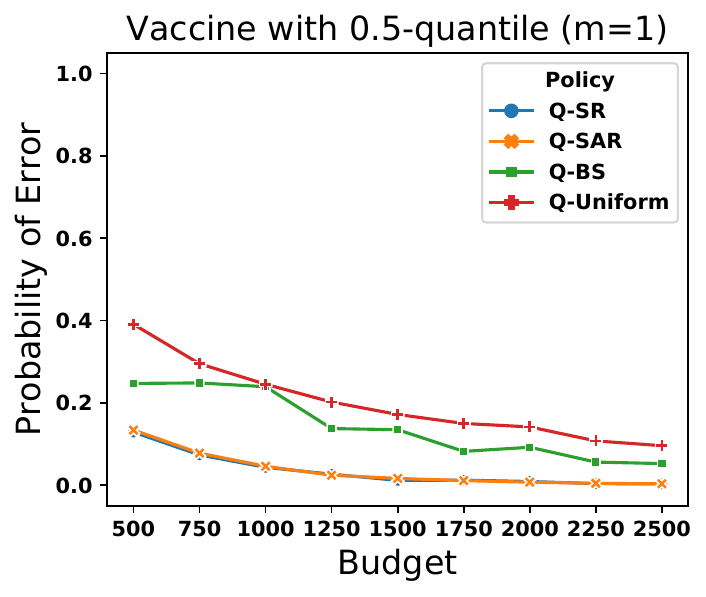}
    \qquad
    \includegraphics[scale=0.3]{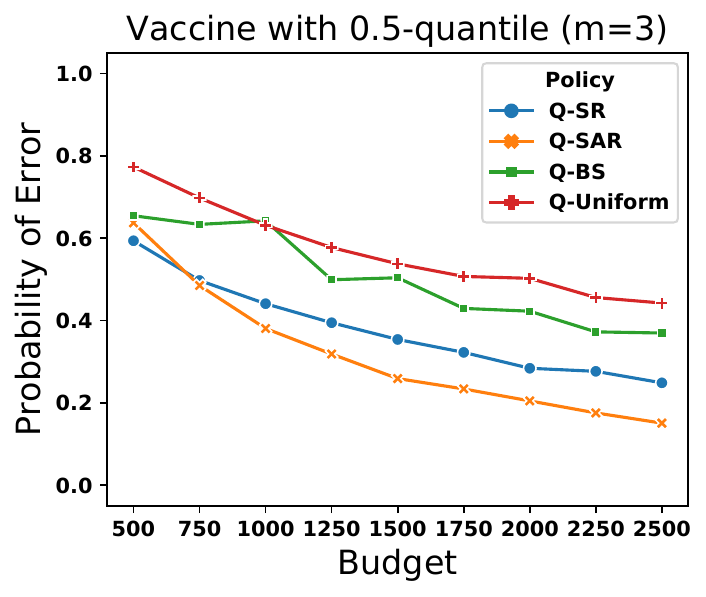}
    \caption{Vaccine BAI Experiments (with 5000 independent runs).
    }
    \label{fig: Vaccine Simulated experiment}
    \vspace{-0.5cm}
  \end{figure}

    \section{Conclusion and Discussion}

    Building on \citet{boucheron2012}, we prove a new concentration inequality of order statistics w.r.t its expectation,
    with which
    we prove a new concentration inequality for quantile estimation w.r.t. the population quantile.
    The new concentration inequalities are two-sided, and work for all distributions with non-decreasing hazard rate (IHR).
    Our assumption of positive (and unbounded) rewards results in asymmetric left and right tail bounds.
    The concentration inequalities for both order statistics and quantiles have convergence rate $\mathcal{O}(\sqrt{\frac{1}{n}})$.
    A larger value of $L$, the lower bound of the hazard rate, results in faster convergence.
    The proposed inequalities may be of independent interest.

    In this paper, we consider the $m$ best arms identification problem with fixed budget.
    Motivated by risk-averse decision-making, the optimal arm set is determined by the $\tau$-quantiles of reward distributions instead of the mean.
    The quantile level $\tau$ provides an additional level of flexibility for modelling, depending on the risk preference.
    We proposed the quantile-based successive accepts and rejects (Q-SAR),
    the first quantile based bandit algorithm for the fixed budget setting.
    We apply our concentration inequality to prove an upper bound on error probability, which is characterised by the problem complexity.
    Empirical results show Q-SAR outperforms baseline algorithms for identifying multiple arms.
    One extension of this work is to allow different sample sizes in Q-SAR, that takes different quantile levels or lower bounds of hazard rate into consideration. 
    Another future work is to derive the matching lower bound of error probability. 
    We hope that this work opens the door towards new bandit approaches for other summary statistics.


    \subsection*{Acknowledgements}
    The authors would like to thank S\'{e}bastien Bubeck for constructive suggestions about valid assumptions,
    Dawei Chen for generating the vaccine data,
    and Russell Tsuchida and Michael Yang for helpful comments.


\bibliography{icml2021_quantile_bandits}
\bibliographystyle{icml2021}

\newpage
\appendix

\onecolumn
\section*{Supplementary Materials for Quantile Bandits for Best Arms Identification}

We show experiment details in Section \ref{sec: experiment details},
the detailed proofs for both concentration inequalities (Section \ref{sec: concentration proof})
bandit task (Section \ref{sec: Q-SAR proof}),
and discussion in Section \ref{sec: discussion appendix}.

\section{Experiments Details}
\label{sec: experiment details}

In this section, we illustrate experiment details, including simulation details (Section \ref{sec: illustractive example appendix}), vaccine allocation strategy description (Section \ref{sec: vaccine simulation appendix}), Q-SR algorithm (Section \ref{sec: Q-SR appendix}).


\subsection{Illustrative Example}
\label{sec: illustractive example appendix}
We provide more details about the environments setting in Section \ref{sec: experiments}.
We consider two distributions which satisfy our assumptions: absolute Gaussian distribution (Definition \ref{defi: AbsGau}),
and exponential distribution (Definition \ref{defi: Exp}).

\begin{defi}[Absolute Gaussian Distribution]
\label{defi: AbsGau}
Given a Gaussian random variable $X$ with mean $\mu$ and variance $\sigma^2$, the random variable $Y = |X|$ has a absolute Gaussian distribution with p.d.f and c.d.f. shown as,
\begin{align}
    f_{AbsGau}(\mu, \sigma^2) = \frac{1}{\sigma \sqrt{2 \pi}} e^{-\frac{(x-\mu)^{2}}{2 \sigma^{2}}}+\frac{1}{\sigma \sqrt{2 \pi}} e^{-\frac{(x+\mu)^{2}}{2 \sigma^{2}}},\\
    F_{AbsGau}(\mu, \sigma^2) = \frac{1}{2}\left[\operatorname{erf}\left(\frac{x+\mu}{\sigma \sqrt{2}}\right)+\operatorname{erf}\left(\frac{x-\mu}{\sigma \sqrt{2}}\right)\right].
\end{align}
where the error function $\text{erf}\left(x\right)= \frac{1}{\sqrt{\pi}} \int_{-x}^{x} e^{-t^2} dt$. We denote the absolute Gaussian distribution random variable with mean $\mu$ and variance $\sigma^2$ as $|\mathcal{N}(\mu, \sigma^2)|$.
When $\mu = 0$, the lower bound of hazard rate $L = \frac{1}{\sigma \sqrt{2 \pi}}$.
\end{defi}

\begin{defi} [Exponential Distribution]
\label{defi: Exp}
With $\theta > 0$, the p.d.f and c.d.f of exponential distribution are defined as
\begin{align}
    \label{Expon PDF}
    f_{Exp}\left(x, \theta\right) &= \theta e^{-\theta x},\\
    \label{expon CDF}
    F_{Exp}\left(x, \theta\right) &= 1 - e^{-\theta x},
\end{align}
We denote the exponential distribution with $\theta$ as $Exp(\theta)$.
The hazard rate for exponential distribution is a constant and equal to $\theta$, i.e. $h(x) = \theta$.
\end{defi}

We design our experimental environments based on three configurations of reward distributions:
A) $|\mathcal{N}(0, 2)|$ B) $|\mathcal{N}(3.5,2)|$  C) Exp(1/4).
The histogram of these three arms is shown below.

\begin{figure}[h]
    \centering
    \includegraphics[scale=0.6]{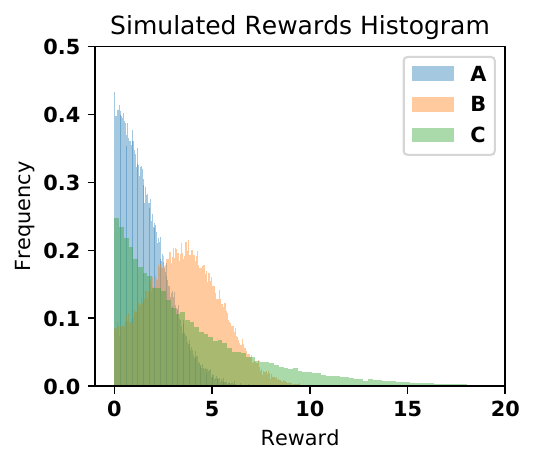}
    \caption{Simulated Arm Rewards Histogram.}
    \label{fig: Simulated Arm Rewards Histogram.}
    \vspace{-0.5cm}
\end{figure}

\subsection{Vaccine Allocation Strategy}
\label{sec: vaccine simulation appendix}
We provide more details about the vaccine allocation strategy in this section.
We allocate 100 vaccine doses (5$\%$ of the population) to 5 age groups
(0-4 years, 5-18 years, 19-29 years, 30-64 years and $>$65 years).
We consider all combinations of groups (resulting in $K=32$ arms),
and denote the allocation scheme as a Boolean 5-tuple, with each position corresponds to the respective age group
(1 represents allocation; 0 otherwise).
We use the median ($\tau=0.5$) as a robust summary statistic for each strategy.
  For the task of identifying the best subset of ages ($m=1$) Q-SAR finds that the
  optimal arm is $(0,1,0,0,0)$, i.e. only allocation to 5-18 years old group.
  For identifying the $m=3$ best arms, the optimal arms are
  $(0,1,0,0,0), (1,1,0,0,0)$ and $(0,1,1,0,0)$,
  indicated as arms 8, 24, and 12 in Figure~\ref{fig: Vaccine Reward Violin Plot.} respectively.

\subsection{Q-SR}
\label{sec: Q-SR appendix}

We extend the Successive Rejects algorithm \cite{audibert2010best} to a quantile version and adapt it to recommend more than one arm.

\begin{algorithm}[h]
    \caption{Q-SR}
    \label{alg:Q-SR}
    \begin{algorithmic}
        \STATE Denote the active set $ \mathcal{A}_1= \{1, ..., K\}$, $\widetilde{\log} (K) = \frac{m}{m+1} + \sum_{p = 1}^{K-m} \frac{1}{K+1-p}, n_0 = 0$, and for $k \in \{1, ..., K-1\},$
        \STATE $$n_p = \left\lceil\frac{1}{\widetilde{\log}(K)} \frac{N-K}{K+1-p}\right\rceil$$
        \STATE For each phase $p = 1,2, ..., K-m$:
        \STATE (1) For each $i \in \mathcal{A}_p$, select arm $i$ for $n_p - n_{p-1}$ rounds.
        \STATE (2) Let $\mathcal{A}_{p+1} = \mathcal{A}_p / \argmin_{i \in \mathcal{A}_p} \hat{Q}^\tau_{i, n_p}$
        \STATE The recommended set is $\mathcal{A}_{K- m + 1}$.
    \end{algorithmic}
    \end{algorithm}

We provide justifications for the design choice of our proposed Q-SR algorithm shown in Algorithm \ref{alg:Q-SR}.
Although that both SR and SAR are analysed on reward distributions with support $[0,1]$,
they can both be directly extended to subgaussian reward distributions \cite{audibert2010best, bubeck2013multiple}.
We propose Quantile-based Successive Rejects (Q-SR), adapted from Successive Rejects (SR) algorithm \citep{audibert2010best}.
To be able to recommend multiple arms, the total phase is designed to be $K - m$ instead of $K-1$,
and the number of pulls for each round is modified to make sure all budgets are used.
More precisely, one is pulled $n_1 =  \left\lceil\frac{1}{\widetilde{\log}(K)} \frac{N-K}{K}\right\rceil$ times,
one is pulled $n_2 =  \left\lceil\frac{1}{\widetilde{\log}(K)} \frac{N-K}{K-1}\right\rceil$ times, $\dots$,
$m+1$ is pulled $n_{K-m} =  \left\lceil\frac{1}{\widetilde{\log}(K)} \frac{N-K}{K+1-(K-m)}\right\rceil$ times, then
\begin{align}
    n_1 + \dots + (m+1) n_{K-m} \leq K + \frac{N-K}{\widetilde{\log}(K)} \left(\frac{m}{m+1} + \sum_{p = 1}^{K-m} \frac{1}{K+1-p}  \right) = N.
\end{align}

As shown in Section \ref{sec: quantile bandits q-sar}, when $m=1$, the Q-SAR algorithm can be reduced to the Q-SR algorithm.
So the theoretical performance of the Q-SR algorithm is guaranteed.
We leave the theoretical analysis of Q-SR for $m > 1$ for the future work.

\section{Concentration Inequality Proof}
\label{sec: concentration proof}


This section shows the proofs of the concentration results shown in Section \ref{sec: concentration inequalities}.
In the following, we will walk through the key statement and show how we achieve our results in details.
For the reader's convenience, we restate our theorems in the main paper whenever needed.
We first introduce the Modified logarithmic Sobolev inequality, which gives the upper bound of the entropy (Eq. (\ref{equ: entropy})) of $\exp(\lambda W)$.

\begin{restatable}[Modified logarithmic Sobolev inequality \citep{ledoux2001concentration}]{theo}{MLSI}
\label{theo: Modified logarithmic Sobolev inequality}
Consider independent random variables $X_1, \dots, X_n$,
let a real-valued random variable $W = f\left(X_{1}, \ldots, X_{n}\right)$, where $f: \mathbb{R}^{n} \rightarrow \mathbb{R}$ is measurable.
Let $W_{i}=f_{i}\left(X_{1}, \ldots, X_{i-1}, X_{i+1}, \ldots, X_{n}\right)$, where $f_{i}: \mathbb{R}^{n-1} \rightarrow \mathbb{R}$ is an arbitrary measurable function.
Let $\phi(x)=\exp(x)-x-1$. Then for any $\lambda \in \mathbb{R}$,
\begin{align}
    \operatorname{Ent}\left[\exp(\lambda W)\right]
    & =\lambda \mathbb{E}\left[W \exp(\lambda W)\right]-\mathbb{E}\left[\exp(\lambda W)\right] \log \mathbb{E}\left[\exp(\lambda W)\right] \\
    & \leq \sum_{i=1}^{n} \mathbb{E}\left[\exp(\lambda W) \phi\left(-\lambda\left(W-W_{i}\right)\right)\right]
\end{align}
\end{restatable}

Consider i.i.d random variables $X_1, \dots, X_n$, and the corresponding order statistics $X_{(1)} \geq \dots \geq X_{(n)}$.
Define the spacing between rank $k$ and $k+1$ order statistics as $S_k = X_{\left(k\right)}-X_{\left(k+1\right)}$.
By taking $W$ as $k$ rank order statistics (or negative $k$ rank), and $W_i$ as nearest possible order statistics, i.e. $k \pm 1$ rank (or negative $k \pm 1$ rank),
Theorem \ref{theo: Modified logarithmic Sobolev inequality} provides the connection between the order statistics and the spacing between order statistics.
The connection is shown in Proposition \ref{prop: new bounds for entropy}.

\OSS*

\begin{proof} 
    We prove the upper bound based on Theorem \ref{theo: Modified logarithmic Sobolev inequality}.
    We first prove Eq. (\ref{equ: new bounds for entropy X}).
    We define $W, W_i$ with $f^k$, $f_i^{k}$ as following.
    Let $W$ be the rank $k$ order statistics of $X_1, \dots, X_n$,
    i.e. $W = f^k(X_{1}, \dots, X_{n}) = X_{(k)}$;
    Let $W_i$ be the rank $k$ order statistics of $X_1, \dots, X_{i-1}, X_{i+1}, \dots, X_n$ (i.e. with $X_i$ removed from $X_1, \dots, X_n$),
    i.e.  $f_i = X_{(k+1)} \mathbb{I}(X_i \geq X_{(k)}) + X_{(k)} \mathbb{I}(X_i < X_{(k)})$.
    So $W_i = X_{(k+1)}$ when the removed element is bigger and equal to $X_{(k)}$, 
    otherwise $W = X_{(k)}$.
    Then the upper bound of $\operatorname{Ent}\left[\exp(\lambda X_{(k)})\right]$ is
    \begin{align}
    &\operatorname{Ent}\left[\exp(\lambda X_{(k)})\right] \nonumber\\
    \leq& \mathbb{E} \big[ \sum_{i=1}^n \exp(\lambda X_{(k)}) \phi\big( -\lambda (X_{(k)} - X_{(k+1)} \mathbb{I}(X_i \geq X_{(k)}) - X_{(k)} \mathbb{I}(X_i < X_{(k)}) \big) \big] && \text{Theorem \ref{theo: Modified logarithmic Sobolev inequality}}\\
    =& \mathbb{E} \big[  \exp(\lambda X_{(k)}) \phi\big( -\lambda (X_{(k)} - X_{(k+1)})  \big) \sum_{i=1}^n \mathbb{I}(X_i \geq X_{(k)}) \big] \\
    = &  k \mathbb{E}\left[\exp(\lambda X_{(k)}) \phi\left(-\lambda S_k \right)\right] \\
    =& k \mathbb{E}\left[\exp(\lambda X_{(k+1)}) \exp(\lambda S_k) \phi\left(-\lambda S_k \right)\right] \\
    =& k \mathbb{E}\left[\exp(\lambda X_{(k+1)}) \zeta\left(\lambda S_k \right)\right] && \zeta(x) = \exp(x) \phi(-x)
    \end{align}

    Similarly, for the proof of Eq. (\ref{equ: new bounds for entropy -X}),
    We define $W, W_i$ with $\widetilde{f}^k$, $\widetilde{f}_i^{k-1}$.
    Let $W$ be the negative value of $k$ rank order statistics of $X_{1}, \dots, X_{n}$, i.e. $W = \widetilde{f}^k(X_{1}, \dots, X_{n}) = - X_{(k)}$;
    Let $W_i$ be the negative value of ${k-1}$ rank order statistics of $X_1, \dots, X_{i-1}, X_{i+1}, \dots, X_n$.
    Thus when $X_{i} \geq X_{(k-1)}$, $W_i = - X_{(k)}$, 
    otherwise $W_i = - X_{(k-1)}$.
    Then by Theorem \ref{theo: Modified logarithmic Sobolev inequality}, we get
    $\operatorname{Ent}\left[\exp(- \lambda X_{(k)})\right] \leq  (n - k + 1) \mathbb{E}\left[\exp(-\lambda X_{(k)}) \phi\left( - \lambda S_{k-1}  \right)\right].$

\end{proof}

Compared with the proof in \citet{boucheron2012}, we do not choose a different initialisation of $W_i$ in terms of the two cases $k \leq n/2$ and $k > n/2$,
which does not influence the concentration rates of empirical quantiles, and allows us to extend the proof to all ranks (excluding extremes).
We derive upper bounds for both $\operatorname{Ent}\left[\exp(\lambda X_{(k)})\right]$ and $\operatorname{Ent}\left[\exp(- \lambda X_{(k)})\right]$,
which allows us to derive two-sided concentration bounds instead of one-sided bound.
Now we show the proof of Theorem \ref{theo: New Extended Exponential Efron-Stein inequality}.



\NewExtendedEES*

\begin{proof}
    The proof of Eq. (\ref{equ: new Exponential Efron-Stein inequality right}) is based on Proposition \ref{prop: new bounds for entropy}
    and follows the same reasoning from \cite{boucheron2012} Theorem 2.9.
    Note since Eq. (\ref{equ: new bounds for entropy X}) holds for $k \in [1,n)$,
    Eq. (\ref{equ: new Exponential Efron-Stein inequality right}) can be proved for $k \in [1,n)$ (\citet{boucheron2012} only proved for $k \in [1,n/2]$).

    We now prove Eq. (\ref{equ: new Exponential Efron-Stein inequality left}).
    Recall $\phi(x) = \exp(x) - x - 1$.
    $\phi(x)$ is nonincreasing when $x \leq 0$ and nondecreasing otherwise.
    By Proposition \ref{prop: new bounds for entropy} and Proposition \ref{prop: new negative association for order statistics and spacing} (which will be shown later), for $\lambda \geq 0$,
    \begin{align}
    \label{equ: new proof for Theorem 1 step 1.1}
    \operatorname{Ent}\left[\exp(-\lambda X_{(k)})\right] \leq & (n - k + 1) \mathbb{E}\left[\exp(-\lambda X_{(k)}) \phi\left( - \lambda S_{k-1}  \right)\right]
    && \text{By Proposition \ref{prop: new bounds for entropy}}\\
    \label{equ: new proof for Theorem 1 step 1.2}
    \leq & (n - k + 1) \mathbb{E}\left[\exp(-\lambda X_{(k)}) \right] \mathbb{E}\left[\phi\left( - \lambda S_{k-1}  \right)\right] && \text{By Proposition \ref{prop: new negative association for order statistics and spacing}}
    \end{align}

    Multiplying both sides by $\exp(\lambda \mathbb{E}[X_{(k)}])$,
    \begin{align}
    \label{equ: new proof for Theorem 1 step 2}
        \operatorname{Ent}[\exp(\lambda Z_k^\prime)] \leq (n-k+1) \mathbb{E}[\exp(\lambda Z_k^\prime)] \mathbb{E}[\phi(- \lambda S_{k-1})].
    \end{align}
    With the fact $\phi(x) \leq \frac{1}{2} x^2$ when $x \leq 0$, and $- \lambda S_{k-1} \leq 0$, we have $\mathbb{E}[\phi(- \lambda S_{k-1})] \leq \frac{\lambda^2}{2} \mathbb{E}[S_{k-1}^2]$.
    We then obtain
    \begin{align}
    \label{equ: new EEFS key step 1}
        \frac{\operatorname{Ent}\left[\exp(\lambda Z_k^\prime)\right]}{\lambda^{2} \mathbb{E}\left[\exp(\lambda Z_k^\prime)\right]} \leq \frac{n-k+1}{\lambda^2} \mathbb{E}[\phi(-\lambda S_{k-1})]
        \leq \frac{n-k+1}{2} \mathbb{E}[S_{k-1}^2].
    \end{align}
    We now solve this integral inequality. Integrating left side, with the fact that $\lim _{\lambda \rightarrow 0} \frac{1}{\lambda} \log \mathbb{E} \exp(\lambda Z_k^\prime)=0$, for $\lambda \geq 0$, we have
    \begin{align}
    \int_0^\lambda \frac{\operatorname{Ent}\left[\exp(t Z_k^\prime)\right]}{t^{2} \mathbb{E}\left[\exp(t Z_k^\prime)\right]} dt
    = \int_0^\lambda \frac{\mathbb{E}[ t Z_k^\prime] - \log \mathbb{E}[\exp(t Z_k^\prime)]}{t^2} dt
    = \frac{\log \mathbb{E}[\exp (t Z_k^\prime)]}{t}|_0^\lambda
    \label{equ: new EEFS key step 2}
    = &\frac{1}{\lambda} \log \mathbb{E}[
        \exp(\lambda Z_k^\prime)].
    \end{align}
    Integrating right side, for $\lambda \geq 0$,
    \begin{align}
        \label{equ: new EEFS key step 3}
        \int_0^\lambda \frac{n-k+1}{2} \mathbb{E}[S_{k-1}^2] dt
        = \frac{\lambda (n-k+1)}{2} \mathbb{E}[S_{k-1}^2].
    \end{align}
    Combining Eq. (\ref{equ: new EEFS key step 1}), (\ref{equ: new EEFS key step 2}) and (\ref{equ: new EEFS key step 3}), we get
    \begin{align}
        \psi_{Z_k^\prime}(\lambda) = \log \mathbb{E}[\exp(\lambda Z_k^\prime)] \leq \frac{\lambda^2 (n-k+1)}{2} \mathbb{E}[S_{k-1}^2].
    \end{align}
    which concludes the proof.

\end{proof}

To further bound the order statistic spacings in expectation, we introduce
the R-transform (Definition \ref{defi: R-transform}) and R\'enyi's representation (Theorem \ref{theo: Renyi's representation}).
In the sequel, $f$ is a monotone function from $(a,b)$ to $(c,d)$,
its generalised inverse $f^{\leftarrow}: (c,d) \rightarrow (a,b)$ is defined by $f^{\leftarrow}(y) = \inf \{x: a<x<b, f(x) \geq y\}$.
Observe that the R-transform defined in Definition \ref{defi: R-transform} is the quantile transformation with respect to the c.d.f of standard exponential distribution, i.e. $F^{\leftarrow}\left(F_{exp}\left(t\right)\right)$.

\begin{defi} [R-transform]
\label{defi: R-transform}
The R-transform of a distribution F is defined as the non-decreasing function on $[0, \infty)$ by $R(t) = \inf\{x: F(x) \geq 1 - \exp(-t)\} = F^{\leftarrow}(1-\exp(-t))$.
\end{defi}

\begin{theo} [\textit{R\'enyi's representation, Theorem 2.5 in \citep{boucheron2012}}]
\label{theo: Renyi's representation}
Let $X_{\left(1\right)} \geq \ldots \geq X_{\left(n\right)}$ be the order statistics of samples from distribution F, $Y_{\left(1\right)} \geq Y_{\left(2\right)} \geq \ldots \geq Y_{\left(n\right)}$ be the
order statistics of independent samples of the standard exponential distribution, then
\begin{align}
{\scriptstyle
    \left(Y_{\left(n\right)}, \ldots, Y_{\left(k\right)}, \ldots, Y_{\left(1\right)}\right)  \stackrel{d}{=}
    \left(\frac{E_{n}}{n}, \ldots, \sum_{i=k}^{n} \frac{E_{i}}{i}, \ldots, \sum_{i=1}^{n} \frac{E_{i}}{i}\right),
}
\end{align}
where $E_{1}, \ldots, E_{n}$ are independent and identically distributed (i.i.d.) standard exponential random variables, and
\begin{align}
    \left(X_{\left(n\right)}, \ldots, X_{\left(1\right)}\right) \stackrel{d}{=} \left(R \left(Y_{\left(n\right)}\right), \ldots, R \left(Y_{\left(1\right)}\right)\right),
\end{align}
where $R\left(\cdot\right)$ is the R-transform defined in Definition \ref{defi: R-transform}, equality in distribution is denoted by $\stackrel{d}{=}$.
\end{theo}

The R\'enyi's representation shows the order statistics of an Exponential random variable are linear combinations of independent Exponentials,
which can be extended to the representation for order statistics of a general continuous $F$ by quantile transformation using R-transform.
The following proposition states the connection between the \textit{IHR} and R-transform.

\begin{prop} [Proposition 2.7 \citep{boucheron2012}]
\label{prop non-increasing hazard rate}
Let F be an absolutely continuous distribution function with hazard
rate h (assuming density exists), the derivative of R-transform is $ R^{\prime}=1 / h\left(R\right)$. Then if the hazard rate h is non-decreasing (Assumption \ref{ass:IHR}), then for all $t > 0$ and $x > 0$, $R\left(t+x\right)-R(t) \leq x / h\left(R(t)\right).$
\end{prop}

We now show Proposition \ref{prop: new negative association for order statistics and spacing} based on
the R\'enyi's representation (Theorem \ref{theo: Renyi's representation}) and Harris' inequality (Theorem \ref{theo: harris' inequality}).
Proposition \ref{prop: new negative association for order statistics and spacing} allows us to upper bound the expectation of multiplication of two functions in terms of
the multiplication of expectation of those two functions respectively.
We will use this property to prove Theorem \ref{theo: New Extended Exponential Efron-Stein inequality}.

\begin{theo}[Harris' inequality \citep{boucheron2013}]
    \label{theo: harris' inequality}
Let $X_{1}, \ldots, X_{n}$ be independent real-valued random variables and define
the random vector $X=\left(X_{1}, \ldots, X_{n}\right)$ taking values in $\mathbb{R}^{n} .$
If $f : \mathbb{R}^{n} \rightarrow \mathbb{R}$ is nonincreasing and $g : \mathbb{R}^{n} \rightarrow \mathbb{R}$ is nondecreasing then
$$
\mathbb{E}[f(X) g(X)] \leq \mathbb{E}[f(X)] \mathbb{E}[g(X)]
$$
\end{theo}

\begin{prop}[Negative Association]
\label{prop: new negative association for order statistics and spacing}
Let the order statistics spacing of rank $k - 1$ as $S_{k-1} = X_{(k-1)} - X_{(k)}$.
Then $ X_{(k)}$ and $S_{k-1}$ are negatively associated: for any pair of non-increasing function $f_1$ and $f_2$,
\begin{align}
    \mathbb{E}\left[f_1(X_{(k)}) f_2\left( S_{k-1} \right)\right] \leq
    \mathbb{E}\left[f_1(X_{(k)})\right] \mathbb{E}\left[f_2\left( S_{k-1} \right)\right].
\end{align}
\end{prop}

\begin{proof}
From Definition \ref{defi: R-transform} and Theorem \ref{theo: Renyi's representation},
let $Y_{(1)}, \ldots, Y_{(n)}$ be the order statistics of an exponential sample.
Let $E_{k-1}=$ $Y_{(k-1)}-Y_{(k)}$ be the ${(k-1)}^{\text {th }}$ spacing of the exponential sample.
By Theorem \ref{theo: Renyi's representation}, $E_{k-1}$ and $Y_{(k)}$ are independent.
\begin{align}
\mathbb{E}\left[f_{1}\left(X_{(k)}\right) f_{2}\left(S_{k-1}\right)\right]
= &\mathbb{E}\left[f_{1}(R(Y_{(k)})) f_{2}\left( R(Y_{(k-1)})-R(Y_{(k)})\right)\right]\\
=&\mathbb{E}\left[\mathbb{E}\left[f_{1}(R(Y_{(k)})) f_{2}\left(R(E_{k-1} + Y_{(k)})-R(Y_{(k)})\right)| Y_{(k)}\right]\right]\\
=&\mathbb{E}\left[f_{1}(R(Y_{(k)})) \mathbb{E}\left[ f_{2}\left(R(E_{k-1} + Y_{(k)})-R(Y_{(k)})\right)| Y_{(k)}\right]\right].
\end{align}
The function $f_1 \circ R$ is non-increasing.
Almost surely, the conditional distribution of $(k-1) E_{k-1}$ w.r.t $Y_{(k)}$ is the exponential distribution.
\begin{align}
    \mathbb{E}\left[ f_{2}\left(R(E_{k-1} + Y_{(k)})-R(Y_{(k)})\right)| Y_{(k)}\right] =
    \int_0^\infty e^{-x} f_2(R(\frac{x}{k-1} + Y_{(k)})- R(Y_{(k)})) d x.
\end{align}
As $F$ is IHR, $R(\frac{x}{k-1} +y) - R(y) = \int_0^{x/(k-1)} R^\prime (y+z)dz$ is non-increasing w.r.t. $y$
(from Proposition \ref{prop non-increasing hazard rate} we know R is concave when F is IHR).
Then $\mathbb{E}\left[ f_{2}\left(R(E_{k-1} + Y_{(k)})-R(Y_{(k)})\right)| Y_{(k)}\right]$ is non-decreasing function of $Y_{(k)}$.
By Harris' inequality,
\begin{align}
    \mathbb{E}\left[f_{1}\left(X_{(k)}\right) f_{2}\left(S_{k-1}\right)\right]
    \leq &\mathbb{E}\left[f_{1}(R(Y_{(k)}))\right] \mathbb{E}\left[\mathbb{E}\left[ f_{2}\left(R(E_{k-1} + Y_{(k)})-R(Y_{(k)})\right)| Y_{(k)}\right]\right]\\
    = & \mathbb{E}\left[f_{1}(R(Y_{(k)}))\right] \mathbb{E}\left[ f_{2}\left(R(E_{k-1} + Y_{(k)})-R(Y_{(k)})\right)\right]\\
    = & \mathbb{E}\left[f_1(X_{(k)})\right] \mathbb{E}\left[f_2(S_{k-1})\right].
\end{align}
\end{proof}

We prove Proposition \ref{prop: bound of expected spacing} in the following by transform the spacing based on R\'enyi's representation
and the property described in Proposition \ref{prop non-increasing hazard rate}.

\BoundExpSpacing*
\begin{proof}
We show the upper bound the expectations of the $k^{th}$ spacing of order statistics, assuming the lower bound hazard rate is L.
The following proof uses Proposition \ref{prop non-increasing hazard rate}, which requires Assumption \ref{ass:IHR} hold.
\begin{align}
    \mathbb{E}[S_k]
    &= \mathbb{E}[X_{(k)} - X_{(k + 1)}] \nonumber\\
    \label{equ: lemma 1 proof 2}
    &= \mathbb{E}[R\left(Y_{\left(k+1\right)} + \frac{E_k}{k}\right) - R\left(Y_{\left(k+1\right)}\right)] && \text{By Theorem \ref{theo: Renyi's representation}}\\
    \label{equ: lemma 1 proof 3}
    &= \int_{Y} \int_{E} \left(R\left(y + \frac{z}{k}\right) - R\left(y\right)\right) f_Y\left(y\right) f_E\left(z\right) \de z \de y \\
    & \leq \int_{Y} \int_{E} \frac{z}{k \times h\left(R\left(y\right)\right)} f_Y\left(y\right) f_E\left(z\right) \de z \de y && \text{By Proposition \ref{prop non-increasing hazard rate}}\\
    \label{equ: lemma 1 proof 4}
    & \leq \int_{E} \frac{z}{kL} f_E\left(z\right) \de z
    = \frac{1}{kL} && \text{$L$ is the lower bound of the hazard rate}.
\end{align}
\end{proof}

Using the same technique of shown in Proposition \ref{prop: bound of expected spacing}, we prove Lemma \ref{lemma: concentration for Z_k} and Lemma \ref{lemma: new concentration for -Z_k} by further bounding inequalities shown in Theorem \ref{theo: New Extended Exponential Efron-Stein inequality}.

\Zk*

\begin{proof}
We first prove Eq. (\ref{equ: Exponential Efron-Stein inequality Z_k}). From Theorem \ref{theo: Renyi's representation}, we can represent the spacing as $S_{k}=X_{\left(k\right)}-X_{\left(k + 1\right)} \stackrel{d}{=} R\left(Y_{\left(k+1\right) + E_{k} /k} \right)-R\left(Y_{\left(k+1\right)}\right)$, where $E_{k}$ is standard exponentially distributed and independent of $Y_{\left(k + 1\right)}$.
The following proof uses Proposition \ref{prop non-increasing hazard rate}, which requires Assumption \ref{ass:IHR} hold.
\begin{align}
    \label{equ:Theo 4 proof 1}
    \psi_{Z_k}(\lambda) \leq & \lambda \frac{k}{2}\mathbb{E}\left[S_{k}\left(\exp(\lambda S_{k})-1\right)\right] && \text{By Theorem \ref{theo: New Extended Exponential Efron-Stein inequality}} \\
    \label{equ:Theo 4 proof 2}
    \leq & \lambda \frac{k}{2}  \int_{E} \int_{Y} \frac{z}{h\left(R\left(y\right)\right)k} \left(\exp( \frac{\lambda z}{h\left(R\left(y\right)\right)k}) - 1\right)
    f_Y\left(y\right) f_E\left(z\right) \de y \de z && \text{By Proposition \ref{prop non-increasing hazard rate}}\\
    \label{equ:Theo 4 proof 3}
    \leq & \frac{k}{2} \int_E \frac{\lambda}{Lk} z \left(\exp(\frac{\lambda}{Lk} z) -1\right) f_E\left(z\right) \de z \\
    \label{equ:Theo 4 proof 4}
    =  & \frac{k}{2} \int_0^\infty \frac{\lambda}{Lk} z \left(\exp(\frac{\lambda}{Lk} z) -1\right) \exp(-z) \de z\\
    \label{equ:Theo 4 proof 5}
    \leq &  \frac{\lambda^{2} v^r}{2(1-c^r \lambda)}, && \text{With } v^r = \frac{2}{k L^2}, c^r = \frac{2}{kL}
\end{align}
The last step is because for $0 \leq \mu \leq \frac{1}{2}$, $\int_{0}^{\infty} \mu z\left(\exp(\mu z)-1\right) \exp(-z) \mathrm{d} z=\frac{\mu^{2}\left(2-\mu\right)}{\left(1-\mu\right)^{2}} \leq \frac{2 \mu^{2}}{1-2 \mu}$.
where we let $\mu = \frac{\lambda}{Lk}$.

From Eq. (\ref{equ: Exponential Efron-Stein inequality Z_k}) to Eq. (\ref{equ: Concentration inequality for Z_k}),
we convert the bound of logarithmic moment generating function to the tail bound by using the Cram\'er-Chernoff method \citep{boucheron2013}.
Markov's inequality implies, for $\lambda \geq 0$,
\begin{align}
    \mathbb{P}(Z_k \geq \gamma) \leq \exp(-\lambda \gamma) \mathbb{E}[\exp(\lambda Z_k)].
\end{align}
To choose $\lambda$ to minimise the upper bound, one can introduce $\psi_{Z_k}^{*}(\gamma)=\sup _{\lambda \geq 0}\left(\lambda \gamma-\psi_{Z_k}(\lambda)\right)$. Then we get $\mathbb{P}(Z_k \geq \gamma) \leq \exp \left(-\psi_{Z_k}^{*}(\gamma)\right)$.
Set $h_1(u) := 1 + u - \sqrt{1+2u}$ for $u > 0$, we have
\begin{align}
\label{proof: bound of Z_k h_1(u)}
    \psi^\ast_{Z_k}(t) = \operatorname{sup}_{\lambda \in (0, 1/c^r)} (\gamma\lambda - \frac{\lambda^{2} v^r}{2(1-c^r \lambda)}) = \frac{v^r}{{(c^r)}^2} h_1(\frac{c^r\gamma}{v^r})
\end{align}
Since $h_1$ is an increasing function from $(0, \infty)$ to $(0, \infty)$ with inverse function
$h_1^{-1}(u) = u + \sqrt{2u}$ for $u >0$, we have $\psi^{\ast -1}(u) = \sqrt{2v^r u} + c^r u$. Eq. (\ref{equ: Concentration inequality for Z_k}) is thus proved.
\end{proof}


\NewNegZk*

\begin{proof}
The proof is similar to the proof of Lemma \ref{lemma: concentration for Z_k}.
From Theorem \ref{theo: Renyi's representation}, we can represent the spacing as $S_{k-1}=X_{\left(k-1\right)}-X_{\left(k\right)} \stackrel{d}{=} R\left(Y_{\left(k\right) + E_{k-1} /(k-1)} \right)-R\left(Y_{\left(k\right)}\right)$, where $E_{k-1}$ is standard exponentially distributed and independent of $Y_{\left(k\right)}$.
The following proof uses Proposition \ref{prop non-increasing hazard rate}, which requires Assumption \ref{ass:IHR} hold.
\begin{align}
    \psi_{Z_k^\prime}(\lambda)
    \leq & \frac{\lambda^2 (n-k+1)}{2} \mathbb{E}[S_{k-1}^2] && \text{By Theorem \ref{theo: New Extended Exponential Efron-Stein inequality}} \\
    \leq & \frac{\lambda^2 (n-k+1)}{2} \int_{Y} \int_{E} \left(\frac{z}{(k-1) \times h\left(R\left(y\right)\right)}\right)^2 f_Y\left(y\right) f_E\left(z\right) \de z \de y && \text{By Proposition \ref{prop non-increasing hazard rate}}\\
    \leq & \frac{\lambda^2 (n-k+1)}{2} \int_{0}^{\infty} \left(\frac{z}{(k-1)L}\right)^2 \exp(-z) \de z \\
    \leq & \frac{\lambda^2 v^l}{2}. && \text{With } v^l = \frac{2(n-k+1)}{(k-1)^2 L^2}
\end{align}

Eq. (\ref{equ: new Exponential Efron-Stein inequality -Z_k}) is proved. Follow the Cram\'er-Chernoff method described above, we can prove Eq. (\ref{equ: new Concentration inequality for -Z_k}).
\end{proof}

The concentration results for order statistics can be of independent interest.
For example, one can take this result and derive the concentration for sum of order statistics
by applying Hoeffding's inequality \citep{hoeffding1994probability} or Bernstein's inequality \citep{bernstein1924modification}.
\citet{kandasamy_parallelised_ts} took the results from \citet{boucheron2012} and showed such results,
but limited for right tail result for exponential random variables of rank 1 order statistics (i.e. maximum).

Now we convert the concentration results of order statistics to the quantiles, based on the results from Lemma
\ref{lemma: concentration for Z_k} and \ref{lemma: new concentration for -Z_k}, and the Theorem \ref{theo: link expected order statistics and population quantile} ,
which shows connection between expected order statistics and quantiles.

\TwosideBoundsQuantiles*

\begin{proof}
    Denote the confidence interval for the right tail bound of order statistics as $d^r_{k, \gamma} = \sqrt{2 v^r \gamma} + c^r \gamma$.
    From Lemma \ref{lemma: concentration for Z_k}, we have $\mathbb{P}\left( X_{(k)} - \mathbb{E}[X_{(k)}] \geq d^r_{k, \gamma}  \right) \leq \exp(-\gamma)$.
    With $k = \lfloor n(1-\tau) \rfloor$, we have $\hat{Q}^\tau_n = X_{(k)}$ and from Theorem \ref{theo: link expected order statistics and population quantile},
    we have $\mathbb{E}[X_{(k)}] \leq Q^\tau + w_n$.
    With probability at least $1- \exp (-\gamma)$, the following event holds
    \begin{align}
         X_{(k)} - \mathbb{E}[X_{(k)}] < d^r_{k, \gamma}
         \Rightarrow X_{(k)} < \mathbb{E}[X_{(k)}] + d^r_{k, \gamma} \leq Q^\tau + w_n + d^r_{k, \gamma}
        \Rightarrow \hat{Q}^\tau_{n} - Q^\tau < w_n + d^r_{k, \gamma},
    \end{align}
    from which we have $\mathbb{P}(\hat{Q}^\tau_{n} - Q^\tau \geq w_n + d^r_{k, \gamma}) \leq \exp(-\gamma)$. 


    Denote the confidence interval for the right tail bound of order statistics as $d^l_{k, \gamma} = \sqrt{2 v^l \gamma}$.
    From Lemma \ref{lemma: new concentration for -Z_k}, we have $\mathbb{P}\left( \mathbb{E}[X_{(k)}] - X_{(k)} \geq d^l_{k, \gamma}  \right) \leq \exp(-\gamma)$.
    With $k = \lfloor n(1-\tau) \rfloor$, we have $\hat{Q}^\tau_n = X_{(k)}$ and from Theorem \ref{theo: link expected order statistics and population quantile},
    we have $\mathbb{E}[X_{(k)}] \geq Q^\tau - w_n$.
    With probability at least $1- \exp (-\gamma)$, the following event holds
    \begin{align}
        \mathbb{E}[X_{(k)}] - X_{(k)} < d^l_{k, \gamma}
        \Rightarrow - X_{(k)} < - \mathbb{E}[X_{(k)}] + d^l_{k, \gamma} \leq - (Q^\tau - w_n)  + d^l_{k, \gamma}
        \Rightarrow  Q^\tau  - \hat{Q}^\tau_{n} < w_n + d^l_{k, \gamma},
    \end{align}
    from which we have $\mathbb{P}( Q^\tau  - \hat{Q}^\tau_{n} \geq w_n + d^l_{k, \gamma}) \leq \exp(-\gamma)$. This concludes the proof.
\end{proof}

In ths following, we show the representations for the concentration results.

\begin{restatable}[Representation of Concentration inequalities for Order Statistics]{coro}{RepConOS}
    \label{coro: repre con os}
    For $\epsilon > 0$, the concentration inequalities for order statistics in Lemma \ref{lemma: concentration for Z_k} and \ref{lemma: new concentration for -Z_k}
    can be represented as
    \begin{align}
        \label{equ: right repre con os}
        & \mathbb{P}\left(X_{(k)} - \mathbb{E}[X_{(k)}] \geq \epsilon \right) \leq \exp\left(-\frac{{\epsilon}^2}{2 (v^r+ c^r \epsilon)}\right),\\
        \label{equ: left repre con os}
        & \mathbb{P}\left(  \mathbb{E}[X_{(k)}] - X_{(k)} \geq \epsilon \right) \leq \exp\left(- \frac{{\epsilon} ^ 2}{2 v^l}\right).
    \end{align}
  \end{restatable}

\begin{proof}
    Eq. (\ref{equ: left repre con os}) follows by setting $\epsilon = \sqrt{2 v^l \gamma}$.
    We now show the case for Eq. (\ref{equ: right repre con os}).
    Recall from the proof of Lemma \ref{lemma: concentration for Z_k} Eq. (\ref{proof: bound of Z_k h_1(u)}), $h_1(u)=1+u-\sqrt{1+2 u}$.
    Follow the elementary inequality
    \begin{align}
        h_1(u) \geq \frac{u^{2}}{2(1+u)} \quad u>0.
    \end{align}
    Lemma \ref{lemma: concentration for Z_k} implies $\psi_{Z_k}^{*}(t) \geq \frac{t^{2}}{2(v+c t)}$,
    so the statement Eq. (\ref{equ: right repre con os}) follows from Chernoff’s inequality.
\end{proof}

Recall that for Q-SAR, we are interested in events of small probability, that is
for large values of $\gamma$ in Theorem~\ref{theo: New Two-side Concentration inequality for quantiles.}.
In the corollary below, we focus on such events of small probability by considering
$\gamma \geq 1$ (i.e. error less than $\frac{1}{e}\approx 0.37$), which allows
a simpler expression.

\RepConQ*

\begin{proof}
    With $\gamma \geq 1$, we have $\gamma w_n \geq w_n$, then with probability at least $1 - \exp(- \gamma)$, we have
    \begin{align}
        \hat{Q}^\tau_{n} - Q^\tau \leq \sqrt{2 v^r \gamma} + c^r \gamma + w_n \leq \sqrt{2 v^r \gamma} + c^r \gamma + w_n \gamma. \nonumber
    \end{align}
    That is, we have
    \begin{align}
        \mathbb{P}\left(\hat{Q}^\tau_{n} - Q^\tau \geq \sqrt{2 v^r \gamma} + (c^r + w_n) \gamma \right) \leq \exp(-\gamma).
    \end{align}
    Similarly, one can prove the other side.
    Then the similar reasoning as shown in proof of Corollary \ref{coro: repre con os} concludes the proof.
\end{proof}

\section{Bandits Proof}
\label{sec: Q-SAR proof}

In this section, we provide the proof for the bandit theoritical result (Section \ref{sec: appendix qsar proof}), with a re-expression of the concentration result (Section \ref{sec: appendix re-express concentration}).


\subsection{Re-expression of Concentration Results}
\label{sec: appendix re-express concentration}

The proof of Q-SAR error bound uses the concentration results for quantiles.
We first further derive the result shown in Corollary \ref{coro: Rep Con Q} to show direct dependency on the number of samples $n$.
Recall the rank $k = \lfloor n(1-\tau) \rfloor$ with quantile level $\tau \in (0,1)$, which can be re-expressed as
\begin{align}
    \label{equ: n range}
    \frac{k}{1-\tau} &\leq n \leq \frac{k+1}{1- \tau}\\
    n(1-\tau) -1 &\leq k \leq n(1-\tau).
\end{align}
We show the representation of concentration depending on $n$ in the following.

\begin{lemma}
    \label{lemma: quantile concentration, dependency on n}
    For $\epsilon > 0$, recall $n$ denotes the number of samples,  $b$ is a constant depending on the density about $\tau$-quantile ($0<\tau<1$), $L$ is the lower bound of hazard rate.
    With $n \geq \frac{4}{1-\tau}$,
    Corollary \ref{coro: Rep Con Q}
    can be represented as
    \begin{align}
        & \mathbb{P}\left(\hat{Q}^\tau_{n} - Q^\tau \geq \epsilon \right) \leq \exp\left(-\frac{n (1-\tau)L^2 {\epsilon }^2}{2 (\frac{8}{3} + \frac{4}{3}(2L + b(1-\tau) L^2) \epsilon )}\right), \nonumber\\
        & \mathbb{P}\left( Q^\tau  - \hat{Q}^\tau_{n} \geq \epsilon \right) \leq \exp\left(-\frac{ n (1-\tau)L^2 {\epsilon }^2}{2 (\frac{4(1+\tau)}{1-\tau} + b (1-\tau) L^2 \epsilon )}\right). \nonumber
    \end{align}
    Combining the two bounds together, we have the two-sided bound shown in the following,
    \begin{align}
        \mathbb{P}\left(|\hat{Q}^\tau_{n} - Q^\tau| \geq \epsilon \right) \leq 2 \exp\left(-\frac{n (1-\tau)L^2 {\epsilon }^2}{2 (\alpha + \beta \epsilon )}\right), \nonumber
    \end{align}
    where $\alpha = \frac{4(1+\tau)}{1-\tau}, \beta = \frac{4}{3}(2 L + b(1-\tau)L^2)$.
\end{lemma}

\begin{proof}
    By assuming $n \geq \frac{4}{1-\tau}$, we have
    \begin{align}
        n(1-\tau) - 1 &= n(1- (\tau + \frac{1}{n})) \geq \frac{3}{4} n (1-\tau).\\
        n(1-\tau) - 2 &= n(1- (\tau + \frac{2}{n})) \geq \frac{1}{2} n (1-\tau).\label{eq:reason-n-geq-4}
    \end{align}

    Recall $v^r = \frac{2}{k L^2}$, $v^l = \frac{2(n-k+1)}{(k-1)^2 L^2}$, $c^r = \frac{2}{k L}, w_n = \frac{b}{n}$, we have
    \begin{align}
         \mathbb{P}\left(\hat{Q}^\tau_{n} - Q^\tau \geq \epsilon \right) &\leq \exp\left(-\frac{{\epsilon }^2}{2 (v^r+ (c^r + w_n) \epsilon )}\right)\\
         & = \exp\left(-\frac{{\epsilon }^2}{2 (\frac{2}{k L^2}+ (\frac{2}{k L} + \frac{b}{n}) \epsilon )}\right)\\
         &\leq \exp\left(-\frac{{\epsilon }^2}{2 (\frac{2}{k L^2}+ (\frac{2}{k L} + \frac{b (1-\tau)}{k}) \epsilon )}\right) && n \geq \frac{k}{1-\tau}\\
        \label{equ: lemma 3 right proof last step with k}
         &= \exp\left(-\frac{k L^2 {\epsilon }^2}{2 (2+ (2L + b (1-\tau)L^2) \epsilon )}\right) \\
         &\leq \exp\left(-\frac{\frac{3}{4}n (1-\tau) L^2 {\epsilon }^2}{2 (2+ (2L + b (1-\tau)L^2) \epsilon )}\right) && k \geq \frac{3}{4} n (1-\tau)\\
         & = \exp\left(-\frac{n (1-\tau) L^2 {\epsilon }^2}{2 (\frac{8}{3}+ \frac{4}{3}(2L + b (1-\tau)L^2) \epsilon )}\right).
    \end{align}
    Similarly,
    \begin{align}
        \mathbb{P}\left( Q^\tau  - \hat{Q}^\tau_{n} \geq \epsilon \right) &\leq \exp\left(- \frac{{\epsilon } ^ 2}{2 (v^l + w_n \epsilon)}\right)\\
        \label{equ: lemma 3 left proof last step with k}
        & = \exp\left(- \frac{{\epsilon } ^ 2}{2 (\frac{2(n-k+1)}{(k-1)^2 L^2} + \frac{b}{n} \epsilon)}\right)\\
        & \leq \exp\left(- \frac{{\epsilon } ^ 2}{2 (\frac{(1+\tau)}{1/4 n(1-\tau)^2 L^2} + \frac{b}{n} \epsilon)}\right)&& k - 1 \geq \frac{n(1-\tau)}{2}\\
        & = \exp\left(- \frac{{1/4 n (1-\tau)^2 L^2  \epsilon } ^ 2}{2 ((1+\tau) + 1/4 b (1-\tau)^2 L^2 \epsilon)}\right)\\
        & = \exp\left(- \frac{{ n (1-\tau) L^2  \epsilon } ^ 2}{2 (\frac{4(1+\tau)}{1-\tau} +  b (1-\tau) L^2 \epsilon)}\right).
    \end{align}

Then let $\alpha = \max\{\frac{8}{3}, \frac{4(1+\tau)}{1-\tau}\} = \frac{4(1+\tau)}{1-\tau}, \beta = \max\{\frac{4}{3}(2 L + b(1-\tau)L^2),  b (1-\tau) L^2\} = \frac{4}{3}(2 L + b(1-\tau)L^2))$,
we have
\begin{align}
    \mathbb{P}\left(|\hat{Q}^\tau_{n} - Q^\tau| \geq \epsilon \right) & = \mathbb{P}\left(\hat{Q}^\tau_{n} - Q^\tau \geq \epsilon \right) + \mathbb{P}\left( Q^\tau  - \hat{Q}^\tau_{n} \geq \epsilon \right) \nonumber\\
    & \leq 2 \exp\left(-\frac{n (1-\tau)L^2 {\epsilon }^2}{2 (\alpha + \beta \epsilon )}\right). \nonumber
\end{align}

\end{proof}

\begin{remark}[Lower bound of sample size]
    In Lemma \ref{lemma: quantile concentration, dependency on n}, we make an assumption about the lower bound of sample size, i.e. $n \geq \frac{4}{1-\tau}$.
    Note that along with the left inequality of  Equation~(\ref{equ: n range}), the lower bound of sample size can be expressed as $n \geq \frac{\max\{k,4\}}{1-\tau}$. When $k\geq4$, we have $n \geq \frac{k}{1-\tau} = \frac{\lfloor n(1-\tau) \rfloor}{1-\tau}$, which holds for all $n\geq 1$.
    This implies that instead of making an assumption about sample size $n$, we could equivalently
    make an assumption about the rank ($k \geq 4$) or the quantile level
    ($\tau \leq 1 - \frac{4}{n}$ with $n \geq 4$).\\
    Note the constant 4 in $n \geq \frac{4}{1-\tau}$ is chosen to have a simpler expression for the concentration bounds,
    one can choose any constant bigger than 2 (such that the term $n(1-\tau) - 2$
    is valid in Equation~(\ref{eq:reason-n-geq-4})).
\end{remark}

We show a variant of Lemma \ref{lemma: quantile concentration, dependency on n} where we
remove the lower bound assumption of the number of samples.
The derived concentration bounds have a constant term, which does not influence the convergence rate in terms of $n$.

\begin{lemma}
    \label{lemma: quantile concentration, dependency on n, no lower bound on n assumption}
    For $\epsilon > 0$, recall $n$ denotes the number of samples,  $b$ is a constant depending on the density about $\tau$-quantile ($0 < \tau < 1$), $L$ is the lower bound of hazard rate.
    Corollary \ref{coro: Rep Con Q}
    can be represented as
    \begin{align}
        & \mathbb{P}\left(\hat{Q}^\tau_{n} - Q^\tau \geq \epsilon \right) \leq \exp\left(-\frac{n (1-\tau)L^2 {\epsilon }^2}{2 (2 + (2L + b(1-\tau) L^2) \epsilon )} + \frac{L^2 {\epsilon }^2}{2 (2 + (2L + b(1-\tau) L^2) \epsilon )}\right), \nonumber\\
        & \mathbb{P}\left( Q^\tau  - \hat{Q}^\tau_{n} \geq \epsilon \right) \leq \exp\left(-\frac{ n (1-\tau)L^2 {\epsilon }^2}{2 (\frac{2(2+\tau)}{1-\tau} + b (1-\tau) L^2 \epsilon )} + \frac{ L^2 {\epsilon }^2}{2 (\frac{1}{2}\frac{\tau}{1-\tau} + \frac{1}{4} b (1-\tau) L^2 \epsilon )}\right). \nonumber
    \end{align}
    Combining the two bounds together, we have the two-sided bound shown in the following,
    \begin{align}
        \mathbb{P}\left(|\hat{Q}^\tau_{n} - Q^\tau| \geq \epsilon \right) \leq 2 \exp\left(-\frac{n (1-\tau)L^2 {\epsilon }^2}{2 (\widetilde{\alpha} + \widetilde{\beta} \epsilon )} + \frac{L^2 {\epsilon }^2}{2 (\widetilde{\alpha} + \widetilde{\beta} \epsilon )}\right), \nonumber
    \end{align}
    where $\widetilde{\alpha} = 2 \frac{\tau + 2}{1-\tau}, \widetilde{\beta} = 2 L + b(1-\tau)L^2$.
\end{lemma}

\begin{proof}

From Eq. (\ref{equ: lemma 3 right proof last step with k}), we have
    \begin{align}
        \mathbb{P}\left(\hat{Q}^\tau_{n} - Q^\tau \geq \epsilon \right)
        &\leq \exp\left(-\frac{k L^2 {\epsilon }^2}{2 (2+ (2L + b (1-\tau)L^2) \epsilon )}\right) \\
        &\leq \exp\left(-\frac{(n (1-\tau)-1) L^2 {\epsilon }^2}{2 (2+ (2L + b (1-\tau)L^2) \epsilon )}\right) && k \geq n (1-\tau)-1\\
        & = \exp\left(-\frac{(n (1-\tau)) L^2 {\epsilon }^2}{2 (2+ (2L + b (1-\tau)L^2) \epsilon )} + \frac{L^2 {\epsilon }^2}{2 (2+ (2L + b (1-\tau)L^2) \epsilon )}\right).
   \end{align}
From Eq. (\ref{equ: lemma 3 left proof last step with k}), we have
\begin{align}
    \mathbb{P}\left( Q^\tau  - \hat{Q}^\tau_{n} \geq \epsilon \right)
    & \leq \exp\left(- \frac{{\epsilon } ^ 2}{2 (\frac{2(n-k+1)}{(k-1)^2 L^2} + \frac{b}{n} \epsilon)}\right)\\
    & \leq \exp\left(- \frac{L^2 {\epsilon } ^ 2}{2 (\frac{2(n\tau +2)}{(n(1-\tau) -2)^2} + \frac{b L^2}{n} \epsilon)}\right)&& k - 1 \geq n(1-\tau)-2\\
    & \leq \exp\left(- \frac{L^2 {\epsilon } ^ 2}{2 (\frac{2(n\tau +2)}{(n(1-\tau))(n(1-\tau) -4)} + \frac{b L^2}{n} \epsilon)}\right)\\
    & = \exp\left(- \frac{(n(1-\tau))(n(1-\tau) -4) L^2 {\epsilon } ^ 2}{2 (2(n\tau +2) + n(1-\tau)^2 b L^2 \epsilon)}\right)\\
    & = \exp\left(- \frac{n(1-\tau) L^2 {\epsilon } ^ 2}{2 (2\frac{\tau +2/n}{1-\tau} + (1-\tau) b L^2 \epsilon)} + \frac{L^2 {\epsilon } ^ 2}{2 (\frac{1}{2}\frac{(n\tau +2)}{n(1-\tau)} + \frac{1}{4} (1-\tau) b L^2 \epsilon)}\right) \\
    & \leq \exp\left(- \frac{n(1-\tau) L^2 {\epsilon } ^ 2}{2 (2\frac{\tau + 2}{1-\tau} +  b(1-\tau) L^2 \epsilon)} + \frac{ L^2 {\epsilon } ^ 2}{2 (\frac{1}{2}\frac{\tau}{1-\tau} + \frac{1}{4}b (1-\tau) L^2 \epsilon)}\right). && n \geq 1 \rightarrow 1/n \leq 1
\end{align}

Then let $\widetilde{\alpha} = 2\max \left\{1, \frac{\tau + 2}{1-\tau}, \frac{1}{4} \frac{\tau}{1-\tau}\right\} = 2 \frac{\tau + 2}{1-\tau}$,
$\widetilde{\beta} = \max \{2L + b (1-\tau)L^2, b(1-\tau)L^2,  \frac{1}{4} b(1-\tau) L^2)\} = 2L + b (1-\tau)L^2 $.
we have
\begin{align}
    \mathbb{P}\left(|\hat{Q}^\tau_{n} - Q^\tau| \geq \epsilon \right) & = \mathbb{P}\left(\hat{Q}^\tau_{n} - Q^\tau \geq \epsilon \right) + \mathbb{P}\left( Q^\tau  - \hat{Q}^\tau_{n} \geq \epsilon \right) \nonumber
    \leq 2 \exp\left(-\frac{n (1-\tau)L^2 {\epsilon }^2}{2 (\widetilde{\alpha} + \widetilde{\beta} \epsilon )}
    + \frac{L^2 {\epsilon }^2}{2 (\widetilde{\alpha} + \widetilde{\beta} \epsilon )}\right). \nonumber
\end{align}
\end{proof}

\begin{remark}[Constant term in concentration bound]
    Note the constant term $\frac{L^2 {\epsilon }^2}{2 (\widetilde{\alpha} + \widetilde{\beta} \epsilon )}$ in Lemma \ref{lemma: quantile concentration, dependency on n, no lower bound on n assumption} is due to the floor operator of the rank $k$, as explained in Eq. (\ref{equ: n range}). This constant term is a bias term coming from estimating the quantile by a single order statistic and is unavoidable without additional assumptions.\\
    By comparing Lemma \ref{lemma: quantile concentration, dependency on n} and
    Lemma \ref{lemma: quantile concentration, dependency on n, no lower bound on n assumption}
    we observe that
    one needs to balance between the constant term, convergence rate, and assumptions to be made.
    For example, Lemma \ref{lemma: quantile concentration, dependency on n} reduces the constant term by assuming a lower bound on the sample size. On one hand, this assumption guarantees there is enough number of samples to have a more accurate estimation; on the other hand, compared with Lemma \ref{lemma: quantile concentration, dependency on n, no lower bound on n assumption}, Lemma \ref{lemma: quantile concentration, dependency on n} has a smaller convergence rate in terms of $n$ (larger parameters $\alpha, \beta$).
\end{remark}

\subsection{Q-SAR Error Bounds}
\label{sec: appendix qsar proof}

In this section, we show the proof of Q-SAR error bounds, based on the concentration results we proposed.
In Theorem \ref{theo: QSAR}, we show the error bound based on Lemma \ref{lemma: quantile concentration, dependency on n} under the assumption of lower bound of budget.
In Theorem \ref{theo: QSAR variant, with constant term}, we release the budget assumption, and show a variant of the error bound based on Lemma \ref{lemma: quantile concentration, dependency on n, no lower bound on n assumption}.
The proof technique follows \citet{bubeck2013multiple}.

\QSAR*

\begin{proof}

Recall we order the arms according to optimality as $o_1, \dots, o_K$ s.t. $Q^\tau_{o_1} \geq \dots \geq Q^\tau_{o_K}$.
The optimal arm set of size $m$ is $\mathcal{S}_m^\ast = \{o_1, \dots, o_m\}$.
In phase $p$, there are $K + 1 - p$ arms inside of the active set $\mathcal{A}_p$,
we sort the arms inside of $\mathcal{A}_p$ and denote them as $\ell_1, \ell_2, \dots, \ell_{K+1-p}$
such that $Q^\tau_{\ell_1} \geq Q^\tau_{\ell_2} \geq \dots \geq Q^\tau_{\ell_{K+1-p}}$.
If the algorithm does not make any error in the first $p-1$ phases (i.e. not reject an arm from optimal set and not accept an arm from non-optimal set), then we have
\begin{align}
\label{equ: assume no errors in p - 1 phases}
    \{\ell_{1}, \ell_{2}, \ldots, \ell_{l_p}\} \subseteq  \mathcal{S}_m^\ast, \quad \{\ell_{l_p +1}, \ldots, \ell_{K+1-p}\} \subseteq \mathcal{K} \backslash  \mathcal{S}_m^\ast.
\end{align}
Additionally, we sort the arms in $\mathcal{A}_p$ according to the empirical quantiles at phase $p$ as $a_{best} (=a_1), a_2, ..., a_{l_p}, a_{l_p+1}, ..., a_{worst} (=a_{K-p+1})$ such that
$\hat{Q}^\tau_{a_{best}, n_p} \geq \hat{Q}^\tau_{a_2, n_p} \geq \dots \geq \hat{Q}^\tau_{a_{worst}, n_p}$.

Consider an event $\xi$,
\begin{align}
    \xi=\{\forall i \in\{1, \ldots, K\}, p \in\{1, \ldots, K-1\}, \left| \hat{Q}^\tau_{i, n_p} - Q^\tau_i \right| < \frac{1}{4} \Delta_{(K+1-p)}\}. \nonumber
\end{align}

Recall $\alpha = \frac{4(1+\tau)}{1-\tau}$,
and we adapt $\beta$ to $\beta_i$ with $i$ indicating the index of arm, that is,
$\beta_i =  \frac{4}{3}(2 L_i + b_i(1-\tau)L_i^2)$.
Recall the sample size for phase $p$ is $n_p = \lceil \frac{N-K}{\overline{\log}(K) (K + 1 -p)} \rceil$.
Based on Lemma \ref{lemma: quantile concentration, dependency on n} and the union bound,
we derive the upper bound of probability for the complementary event $\bar{\xi}$ as
\begin{align}
    \mathbb{P}(\bar{\xi}) &\leq \sum_{i=1}^K \sum_{p=1}^{K-1} \mathbb{P}\left(\left|\hat{Q}^\tau_{i, n_p} - Q^\tau_i \right| \geq \frac{1}{4} \Delta_{(K+1-p)}\right)
    && \text{union bound}\\
    &\leq \sum_{i=1}^K \sum_{p=1}^{K-1}  2 \exp\left(-\frac{n_p (1-\tau)L_i^2 {(\frac{1}{4} \Delta_{(K+1-p)})}^2}{2 (\alpha + \beta_i \frac{1}{4} \Delta_{(K+1-p)} )}\right) && \text{By Lemma } \ref{lemma: quantile concentration, dependency on n}\\
    &\leq \sum_{i=1}^K \sum_{p=1}^{K-1}   2\exp\left(-\frac{N-K}{\overline{\log}(K) (K+1-p)} \frac{ 1}{\frac{8}{1-\tau}(\frac{4\alpha}{L_i^2  \Delta_{(K+1-p)}^2} + \frac{\beta_i}{L_i^2  \Delta_{(K+1-p)}} )}\right)\\
    &\leq 2K^2 \exp\left( - \frac{N-K}{\overline{\log}(K) H^\tau}\right).
\end{align}
where $H^\tau = \max_{\{i, j \in \mathcal{K}\}} \frac{8 j}{1-\tau}(\frac{4\alpha}{L_i^2  \Delta_{(j)}^2} + \frac{\beta_i}{L_i^2  \Delta_{(j)}})$.

Note that we have assumed the number of samples of each arm is at least $\frac{4}{1-\tau}$ in Lemma \ref{lemma: quantile concentration, dependency on n}.
That is, $ K n_1 = K \lceil \frac{N-K}{\overline{\log}(K) K} \rceil \geq \frac{4}{1-\tau}$, which gives $N \geq \frac{4}{1-\tau} \overline{\log}(K) + K$.
This means, the bound derived above holds when we have budget $N$ no less than $\frac{4}{1-\tau} \overline{\log}(K) + K$.

We show that on event $\xi$, Q-SAR algorithm does not make any error by induction on phases.
Assume that the algorithm does not make any error on the first $p-1$ phases, i.e.  does not reject an arm from optimal set and not accept an arm from non-optimal set.
Then in the following, we show the algorithm does not make an error on the $p^{th}$ phase.
We discuss in terms of two cases:

\textbf{Case 1:} If an arm $\ell_j$ is accepted, then $\ell_j \in  \mathcal{S}_m^\ast$.\\
We prove by contradiction. Assume arm $\ell_j$ is accepted in phase $p$, but $\ell_j \notin  \mathcal{S}_m^\ast$, i.e. $Q^\tau_{\ell_j} \leq Q^\tau_{\ell_{l_p + 1}} \leq Q^\tau_{o_{m+1}}$.
According to Algorithm \ref{alg:Q-SAR}, arm $\ell_j$ is accepted only if its empirical quantile is the maximum among all active arms in phase $p$, thus $\hat{Q}^\tau_{\ell_j, n_p} \geq \hat{Q}^\tau_{\ell_1, n_p}$. On event $\xi$, we have
\begin{align}
& Q^\tau_{\ell_j} + \frac{1}{4} \Delta_{(K+1-p)} > \hat{Q}^\tau_{\ell_j, n_p} \geq \hat{Q}^\tau_{\ell_1, n_p}
> Q^\tau_{\ell_1} - \frac{1}{4} \Delta_{(K+1-p)} \\
\label{equ: proof QSAR term one of Delta > max}
&  \Rightarrow \Delta_{(K+1-p)} > \frac{1}{2} \Delta_{(K+1-p)} > Q^\tau_{\ell_1} - Q^\tau_{\ell_j} \geq Q^\tau_{\ell_1} - Q^\tau_{o_{m+1}}.
\end{align}
Another requirement to accept $\ell_j$ is $\widehat{\Delta}_{best} > \widehat{\Delta}_{worst}$, that is,
\begin{align}
\label{equ: proof QSAR delta best > delta worst}
    \hat{Q}^\tau_{\ell_{j}, n_p} - \hat{Q}^\tau_{a_{l_p + 1}, n_p} > \hat{Q}^\tau_{a_{l_p}, n_p} - \hat{Q}^\tau_{a_{K+1-p}, n_p}.
\end{align}
In the following, we will connect Eq. (\ref{equ: proof QSAR delta best > delta worst}) with the corresponding population quantiles on event $\xi$.
We first connect $\hat{Q}^\tau_{a_{K+1-p}, n_p}$ and $Q^\tau_{\ell_{K+1-p}}$. Since $\hat{Q}^\tau_{a_{K+1-p}, n_p}$ is the minimum empirical quantile at phase $p$,
\begin{align}
\label{equ: proof QSAR case 1 connect a_K+1-p}
    \hat{Q}^\tau_{a_{K+1-p}, n_p} \leq \hat{Q}^\tau_{\ell_{K+1-p}, n_p} < Q^\tau_{\ell_{K+1-p}} + \frac{1}{4} \Delta_{(K+1-p)}.
\end{align}
We then connect $\hat{Q}^\tau_{a_{l_p + 1}, n_p}, \hat{Q}^\tau_{a_{l_p}, n_p}$ to $Q^\tau_{o_m}$.
On event $\xi$, for all $i \leq l_p$,
\begin{align}
    \hat{Q}^\tau_{\ell_{i}, n_p} > Q^\tau_{\ell_i} - \frac{1}{4} \Delta_{(K+1-p)} \geq Q^\tau_{\ell_{l_p}} - \frac{1}{4} \Delta_{(K+1-p)} \geq Q^\tau_{o_m} - \frac{1}{4} \Delta_{(K+1-p)},
\end{align}
which means there are $l_p$ arms in active set, i.e.  $\{\ell_1, \ell_2, \dots, \ell_{l_p}\}$, having empirical quantiles bigger or equal than $Q^\tau_{o_m} - \frac{1}{4} \Delta_{(K+1-p)}$.
Additionally, although $j > l_p$, $\ell_j$ has the maximum empirical quantile, which is bigger than $Q^\tau_{o_m} - \frac{1}{4} \Delta_{(K+1-p)}$ as well.
So in total there are $l_p + 1$ arms having empirical quantiles bigger or equal to $Q^\tau_{o_m} - \frac{1}{4} \Delta_{(K+1-p)}$, i.e.
\begin{align}
\label{equ: proof QSAR case 1 connect a_lp and a_lp+1}
    \hat{Q}^\tau_{a_{l_p}, n_p} \geq \hat{Q}^\tau_{a_{l_p + 1}, n_p} \geq Q^\tau_{o_m} - \frac{1}{4} \Delta_{(K+1-p)}.
\end{align}
Combine Eq. (\ref{equ: proof QSAR delta best > delta worst})(\ref{equ: proof QSAR case 1 connect a_K+1-p})(\ref{equ: proof QSAR case 1 connect a_lp and a_lp+1}) together, we have
\begin{align}
    &(Q^\tau_{\ell_j} + \frac{1}{4} \Delta_{(K+1-p)}) - (Q^\tau_{o_m} - \frac{1}{4} \Delta_{(K+1-p)})
    > (Q^\tau_{o_m} - \frac{1}{4} \Delta_{(K+1-p)}) - (Q^\tau_{\ell_{K+1-p}} + \frac{1}{4} \Delta_{(K+1-p)})\\
\label{equ: proof QSAR term two of Delta > max}
    & \Rightarrow \Delta_{(K+1-p)} > 2 Q^\tau_{o_m} - (Q^\tau_{\ell_j} + Q^\tau_{\ell_{K+1-p}}) >  Q^\tau_{o_m} -  Q^\tau_{\ell_{K+1-p}}.
\end{align}

From Eq. (\ref{equ: proof QSAR term one of Delta > max})(\ref{equ: proof QSAR term two of Delta > max}), we have
$
     \Delta_{(K+1-p)} > \max \{Q^\tau_{\ell_1} - Q^\tau_{o_{m+1}}, Q^\tau_{o_m} - Q^\tau_{\ell_{K+1-p}}\},
$
which contradicts the fact that $\Delta_{(K+1-p)} \leq \max \{Q^\tau_{\ell_1} - Q^\tau_{o_{m+1}}, Q^\tau_{o_m} - Q^\tau_{\ell_{K+1-p}}\}$,
since at phase $p$, there are only $p-1$ arms have been accepted or rejected.
So we have if an arm $\ell_j$ is accepted, then $\ell_j \in  \mathcal{S}_m^\ast$, which finishes the proof of Case 1.

\textbf{Case 2:} If an arm $\ell_j$ is rejected, the $\ell_j \notin  \mathcal{S}_m^\ast$.

The proof of Case 2 is similar to the proof of Case 1.
We prove by contradiction.
Assume arm $\ell_j$ is rejected in phase $p$, but $\ell_j \in  \mathcal{S}_m^\ast$, i.e. $Q^\tau_{\ell_j} \geq Q^\tau_{\ell_{l_p}} \geq Q^\tau_{o_m}$.
According to Algorithm \ref{alg:Q-SAR}, arm $\ell_j$ is rejected only if its empirical quantile is the minimum among all active arms in phase $p$, thus $\hat{Q}^\tau_{\ell_j, n_p} \leq \hat{Q}^\tau_{\ell_{K+1-p}, n_p}$. On event $\xi$, we have
\begin{align}
& Q^\tau_{\ell_j} - \frac{1}{4} \Delta_{(K+1-p)} < \hat{Q}^\tau_{\ell_j, n_p} \leq \hat{Q}^\tau_{\ell_{K+1-p}, n_p}
< Q^\tau_{\ell_{K+1-p}} + \frac{1}{4} \Delta_{(K+1-p)} \\
\label{equ: proof QSAR case 2 term one of Delta > max}
&  \Rightarrow \Delta_{(K+1-p)} > \frac{1}{2} \Delta_{(K+1-p)}  >  Q^\tau_{\ell_j} - Q^\tau_{\ell_{K+1-p}} \geq  Q^\tau_{o_m} -Q^\tau_{\ell_{K+1-p}}.
\end{align}
Another requirement to accept $\ell_j$ is $\widehat{\Delta}_{best} \leq \widehat{\Delta}_{worst}$, i.e.
\begin{align}
\label{equ: proof QSAR delta best <= delta worst}
    \hat{Q}^\tau_{a_1, n_p} - \hat{Q}^\tau_{a_{l_p + 1}, n_p} \leq \hat{Q}^\tau_{a_{l_p}, n_p} - \hat{Q}^\tau_{\ell_{j}, n_p}.
\end{align}
In the following, we will connect Eq. (\ref{equ: proof QSAR delta best <= delta worst}) with the corresponding population quantiles on event $\xi$.
We first connect $\hat{Q}^\tau_{a_{1}, n_p}$ and $Q^\tau_{\ell_{1}}$. Since $\hat{Q}^\tau_{a_{1}, n_p}$ is the maximum empirical quantile at phase $p$,
\begin{align}
\label{equ: proof QSAR case 2 connect a_1}
    \hat{Q}^\tau_{a_{1}, n_p} \geq \hat{Q}^\tau_{\ell_{1}, n_p} > Q^\tau_{\ell_{1}} - \frac{1}{4} \Delta_{(K+1-p)}.
\end{align}
We then connect $\hat{Q}^\tau_{a_{l_p + 1}, n_p}, \hat{Q}^\tau_{a_{l_p}, n_p}$ to $Q^\tau_{o_{m+1}}$.
On event $\xi$, for all $i \geq l_p + 1$,
\begin{align}
    \hat{Q}^\tau_{\ell_{i}, n_p} < Q^\tau_{\ell_i} + \frac{1}{4} \Delta_{(K+1-p)} \leq Q^\tau_{\ell_{l_p +1}} + \frac{1}{4} \Delta_{(K+1-p)} \leq Q^\tau_{o_{m+1}} + \frac{1}{4} \Delta_{(K+1-p)},
\end{align}
Additionally, although $j < l_p + 1$, $\ell_j$ has the minimum empirical quantile, which is smaller than $Q^\tau_{o_{m+1}} + \frac{1}{4} \Delta_{(K+1-p)}$ as well.
So that,
\begin{align}
\label{equ: proof QSAR case 2 connect a_lp and a_lp+1}
    \hat{Q}^\tau_{a_{l_p + 1}, n_p} \leq \hat{Q}^\tau_{a_{l_p}, n_p} \leq Q^\tau_{o_{m+1}} + \frac{1}{4} \Delta_{(K+1-p)}.
\end{align}
Combining Eq. (\ref{equ: proof QSAR delta best <= delta worst}), (\ref{equ: proof QSAR case 2 connect a_1}) and (\ref{equ: proof QSAR case 2 connect a_lp and a_lp+1}) together, we have
\begin{align}
    &(Q^\tau_{\ell_{1}} - \frac{1}{4} \Delta_{(K+1-p)}) - (Q^\tau_{o_{m+1}} + \frac{1}{4} \Delta_{(K+1-p)})
    \leq (Q^\tau_{o_{m+1}} + \frac{1}{4} \Delta_{(K+1-p)}) - (Q^\tau_{\ell_{j}} - \frac{1}{4} \Delta_{(K+1-p)})\\
\label{equ: proof QSAR case 2 term two of Delta > max}
    & \Rightarrow \Delta_{(K+1-p)} \geq  (Q^\tau_{\ell_j} + Q^\tau_{\ell_{1}}) - 2 Q^\tau_{o_{m+1}} >  Q^\tau_{\ell_{1}} - Q^\tau_{o_{m+1}}.
\end{align}

From Eq. (\ref{equ: proof QSAR case 2 term one of Delta > max})(\ref{equ: proof QSAR case 2 term two of Delta > max}), we have
$
     \Delta_{(K+1-p)} > \max \{Q^\tau_{\ell_1} - Q^\tau_{o_{m+1}}, Q^\tau_{o_m} - Q^\tau_{\ell_{K+1-p}}\},
$
which contradicts the fact that $\Delta_{(K+1-p)} \leq \max \{Q^\tau_{\ell_1} - Q^\tau_{o_{m+1}}, Q^\tau_{o_m} - Q^\tau_{\ell_{K+1-p}}\}$,
since at phase $p$, there are only $p-1$ arms have been accepted or rejected.
So we have if an arm $\ell_j$ is rejected, then $\ell_j \notin  \mathcal{S}_m^\ast$, which finishes the proof of Case 2.
\end{proof}

Now we show a variant of Theorem \ref{theo: QSAR}, using the result of Lemma \ref{lemma: quantile concentration, dependency on n, no lower bound on n assumption}.
Define a slightly different problem complexity $\widetilde{H}^\tau$, which has the same form of $H^\tau$ in Eq. (\ref{equ: problem complexity H}) with smaller parameters $\widetilde{\alpha}$ and $\widetilde{\beta}$.
\begin{align}
    \label{equ: problem complexity variant}
    \widetilde{H}^\tau := \max_{i,j \in \mathcal{K}} \frac{8 j}{1-\tau} (\frac{4 \widetilde{\alpha}}{L_i^2 \Delta_{(j)}^2} + \frac{\widetilde{\beta}_i}{L_i^2 \Delta_{(j)}}),
\end{align}
where $\widetilde{\alpha} = 2 \frac{\tau+2}{1-\tau}, \widetilde{\beta} = 2 L_i+ (1-\tau) b_i L_i^2$.

\begin{restatable}[Q-SAR Probability of Error Upper Bound Variant]{theo}{QSARVariant}
    \label{theo: QSAR variant, with constant term}
    For the problem of identifying $m$ best arms out of $K$ arms, the probability of error (Definition \ref{defi: probability of error}) for Q-SAR satisfies
    \begin{align}
      e_N \leq 2 K^2 \exp\left(- \frac{N-K}{\overline{\log}(K) \widetilde{H}^\tau} + C \right), \nonumber
    \end{align}
    where problem complexity variant $\widetilde{H}^\tau$ is defined in Eq. (\ref{equ: problem complexity variant}), and constant $C = \max_{\{i,j \in \mathcal{K}\}} \frac{L_i ^2\Delta_{(j)}^2 }{8(4 \widetilde{\alpha} + \widetilde{\beta}_i \Delta_{(j)})}$.
\end{restatable}

\begin{proof}
The only difference of the proof of Theorem \ref{theo: QSAR} is we derive the bound of of $\mathbb{P}(\bar{\xi})$ (See Eq. (\ref{equ: event xi}) for the definition of event $\xi$) based on a Lemma \ref{lemma: quantile concentration, dependency on n, no lower bound on n assumption}, and we do not have the lower bound assumption made for budget.
\begin{align}
    \mathbb{P}(\bar{\xi}) &\leq \sum_{i=1}^K \sum_{p=1}^{K-1} \mathbb{P}\left(\left|\hat{Q}^\tau_{i, n_p} - Q^\tau_i \right| \geq \frac{1}{4} \Delta_{(K+1-p)}\right)
    && \text{union bound}\\
    &\leq \sum_{i=1}^K \sum_{p=1}^{K-1}   \exp\left(-\frac{n_p (1-\tau)L_i^2 {(\frac{1}{4} \Delta_{(K+1-p)})}^2}{2 (\widetilde{\alpha} + \widetilde{\beta}_i \frac{1}{4} \Delta_{(K+1-p)} )} +
    \frac{L_i^2 {(\frac{1}{4} \Delta_{(K+1-p)})}^2}{2 (\widetilde{\alpha} + \widetilde{\beta}_i \frac{1}{4} \Delta_{(K+1-p)} )}
    \right) && \text{By Lemma } \ref{lemma: quantile concentration, dependency on n, no lower bound on n assumption}\\
    &\leq 2K^2 \exp\left( - \frac{N-K}{\overline{\log}(K) \widetilde{H}^\tau} + C\right),
\end{align}
where $C = \max_{\{i,j \in \mathcal{K}\}} \frac{L_i^2 \Delta_{(j)}^2 }{8(4 \widetilde{\alpha} + \widetilde{\beta}_i \Delta_{(j)})}$.

Then we conclude the proof by following the same reasoning in the proof of Theorem \ref{theo: QSAR}.
\end{proof}

\section{Discussion}
\label{sec: discussion appendix}

\textbf{Quantile Estimation Complexity:}
This paper focuses on how quantiles provide a different way to summarise the distribution of each arm.
Quantiles are interesting and useful summary statistics for risk-averse decision-making,
but estimating quantiles may be more expensive than estimating the mean.
We provide the time complexity of our algorithms (for $K$ arms) in the following.
Estimating quantiles needs binary search in each round when we get new samples.
For Q-SAR, in each phase $p \in [1, K-1]$, the time complexity is $\mathcal{O}(K \log (n_p - n_{p-1}) + K \log K) = \mathcal{O}(K \log (N/K^2) + K \log K)$.
Combining for all $K-1$ phases, the time complexity is  $\mathcal{O}(K^2 \log (N/K^2) + K^2 \log K)$.
For space complexity, one needs to save the samples for each arm and also updates information (quantiles) for each arm, so the space complexity is $\mathcal{O}(N + K)$ for both algorithms.
One could save time and space for estimating quantiles by using online algorithms.
For example, instead of performing binary search from scratch, one can retain
an estimate of the quantile and update the estimate given the new samples.
This is the key idea of online algorithms such as stochastic gradient descent.
Such approaches (and their analysis) is beyond the scope of this paper.

\textbf{Understanding IHR Distributions}:
The hazard rate of random variables provide an useful way to think about real phenomena.
For example, let the random variable $X$ denote the age of a car when it has a serious engine problem for the first time.
One would expect the hazard rate increases over time.
If the random variable $X$ denotes the time before you win a lottery, then the hazard rate would be approximately constant. \\
Some examples of general distributions with IHR include:
\begin{itemize}
    \item Gamma distribution with two parameters $\lambda > 0, \alpha >0$, with p.d.f. $f(x) = \frac{\lambda^\alpha x^{\alpha -1} \exp\{-\lambda x\}}{\Gamma(\alpha)}$ for $x > 0$, where the gamma function $\Gamma(\alpha) = \int_0^\infty x^{\alpha - 1}\exp\{-x\} d x$.
    When $\alpha \geq 1$, the hazard rate is non-decreasing. The case $\alpha = 1$ corresponds to the exponential distribution (Definition \ref{defi: Exp}) and the hazard rate is constant.
    \item Weibull Distribution with two parameters $\lambda > 0$, $p > 0$, with p.d.f. $f(x) = p \lambda^p x^{p-1} \exp\{-{(\lambda x)}^p\}$ for $x > 0$.
    When $p \geq 1$, the hazard rate is non-decreasing. The Weibull distribution reduces to the exponential distribution (Definition \ref{defi: Exp}) when $p = 1$.
    \item Absolute Gaussian distribution (Definition \ref{defi: AbsGau}). The lower bound of hazard rate for the centered Absolute Gaussian distribution is $\frac{1}{\sigma \sqrt{2 \pi}}$.
\end{itemize}
Recall that the IHR assumption allows us to consider distributions of unbounded rewards.
It does so by constraining the tails of the density. The random variable with IHR is light-tailed, i.e. having tails the same as or lighter than an exponential distribution.
Light-tailed distributions include a wide range of distributions, including sub-gamma and sub-Gaussian distributions.

\textbf{Estimate Lower Bound of Hazard Rate and Concentration Inequality}
In practice, one can estimate the lower bound of hazard rate L by estimating the p.d.f. and c.d.f. around 0 (with non-negative support and IHR assumption).
So one can design a UCB-type of algorithms by adaptively estimating L.
But in practice, introducing new variables to estimate would influence the stability of the algorithm.

\section{Post-publication update: details of estimating quantiles from finite data}


\begin{remark}[Single order statistic as quantiles estimator]
\label{remark: Single order statistic as quantiles estimator}
We focus on the single order statistics estimation in our paper for simplicity. 
More complicated methods, such as using multiple order statistics, smoothing between multiple order statistics can improve the estimator, which is out of scope for this paper.  

There are many ways of estimating quantiles from a single order statistic as discussed in \citet{zielinski_quantile_est_survey}. However, no single option of rank is optimal and estimating quantiles from a single order statistic with finite number of samples is biased.  
For example, a less biased option compared to what we used in Eq. (\ref{equ: empirical quantile defi}) one could choose
\begin{align}
    \label{equ: quantile estimator range}
    \hat{Q}^\tau_n = X_{(k)}, \qquad \text{where } k \in (n(1-\tau), n(1-\tau) + 1].
\end{align}
Here we choose the right closed interval to be consistent with our (left-continuous) empirical quantile and (non-increasing) order statistics definition. 
The estimator in Eq. (\ref{equ: quantile estimator range}) is equivalent to 
\begin{align}
    \hat{Q}^\tau_n = X_{(\lfloor n(1-\tau) + 1 \rfloor)},
\end{align}
as used in \cite{tran-thanh_functional_2014,yu2013sample,torossian_x-armed_2019} (they used the $\lceil np \rceil$, where $p$ is the rank for non-decreasing order statistics).
One can switch to this choice of rank, and the subsequent notions of $k, \alpha, \beta$ in Eq. (\ref{equ: empirical quantile defi}) and (\ref{equ: problem complexity H}) and Theorem \ref{theo: New Two-side Concentration inequality for quantiles.} will change accordingly. The concentration and regret rate in terms of $n$ will remain as claimed in this paper. 
\end{remark}


\begin{remark}[Discussion of Theorem \ref{theo: link expected order statistics and population quantile}]
The proof details of Theorem \ref{theo: link expected order statistics and population quantile} follows \citet[Section 4.6]{david_order_1981}. Note our definition of order statistic is non-decreasing, while \citet{david_order_1981}'s is non-increasing.
For a positive integer $k \in [1,n]$, let $U_{(k)}$ be the $k$-th order statistic from a sample of size $n$ from a uniform $(0,1)$ distribution.
For a quantile function $Q(p)$, and the $k$ rank order statistic $X_{(k)}$,
we have $X_{(k)} = Q^{U_{(k)}}$. Expanding $Q^{U_{(k)}}$ in a Taylor series about $p_k = \mathbb{E}[U_{(k)}] = \frac{n-k+1}{n+1}$, gives
\begin{align}
    X_{(k)} \approx Q^{p_k}+\left(U_{(k)}-p_k\right) Q^{\prime}\left(p_k\right)+\frac{1}{2}\left(U_{(k)}-p_k\right)^2 Q^{\prime \prime}\left(p_k\right) 
        +\frac{1}{6}\left(U_{(k)}-p_k\right)^3 Q^{\prime \prime \prime}\left(p_k\right)+\cdots
\end{align}
where $Q^{\prime}(p_k), Q^{\prime \prime}\left(p_k\right), Q^{\prime \prime \prime}\left(p_k\right)$ are the first, second, third derivatives evaluated at $p_k$. 
Let $q_k = 1-p_k = \frac{k}{n+1}$, 
we obtain to order $(n+2)^{-2}$,
\begin{align}
    \label{proof: xr to qpr}
    \mathbb{E}[X_{(k)}] \approx Q^{p_k} + \frac{p_k q_k}{2(n+2)} Q^{\prime \prime}(p_k)+\frac{p_k q_k}{(n+2)^2}\left[\frac{1}{3}\left(q_k-p_k\right) Q^{\prime \prime \prime}(p_k)\right.
        \left.+\frac{1}{8} p_k q_k Q^{\prime \prime \prime \prime}(p_k)\right].
\end{align} 
Now let $k = \lfloor n(1-\tau) \rfloor$, and $\delta \geq \frac{|Q^{p_k} - Q^{\tau}|}{|p_k - \tau|}$, we have
\begin{align}
    \label{proof: qpr to qtau}
    |Q^{p_k} - Q^{\tau}| \leq \delta |p_k - \tau| \in   \left[\frac{\delta k}{n(n+1)}, \frac{\delta (k + n +1)}{n(n+1)} \right)
\end{align}
where we obtain the interval on the right side of Eq. (\ref{proof: qpr to qtau}) by observing that $\tau \in (\frac{n-k-1}{n}, \frac{n-k}{n}]$. The ratio between $|Q^{p_k} - Q^{\tau}|$ and $|p_k - \tau|$ is illustrated in Figure \ref{fig: theo1 ratio}. 
From Eq. (\ref*{proof: xr to qpr}) and (\ref{proof: qpr to qtau}),
by Taylor's Theorem, there exists a constant $b \geq 0$ such that
\begin{align}
 |\mathbb{E}[X_{(k)}] - Q^\tau | \leq \frac{b}{n}.
\end{align}
\textbf{The choice of $\delta$:} The value of $\delta$ is dependent on the smoothness of the true quantile function at the quantile value $\tau$. We leave the choice of $\delta$ as future work. For example, $\delta$ can be defined with respect to the interval, e.g. for a differentiable quantile function $Q(p)$, define $\delta = \max_{p \in (\tau, p_k)} Q^{\prime}(p)$ where $Q^{\prime}(p)$ is the gradient of $Q(p)$. Then $\delta$ depends on $p_k$, which depends on how many samples we have -- intuitively, the more samples we have, the closer $p_k$ and $\tau$ are, and the closer $\delta$ and $Q^{\prime}(\tau)$ are. Then one needs to prove the relationship between $\delta$ and $k$ (and hence $n$). Alternatively, we could make an extra smoothness assumption about the quantile function, e.g. Lipschitz smoothness. Then $\delta$ is independent of $n$ in this case, however, $\delta$ might be a loose bound of the actual ratio we need. e.g. $\delta$ can be large due to the large derivative at right tail of the absolute Gaussian distribution.
\end{remark}

\begin{figure}[ht]
    \centering
    \includegraphics[scale = 0.4]{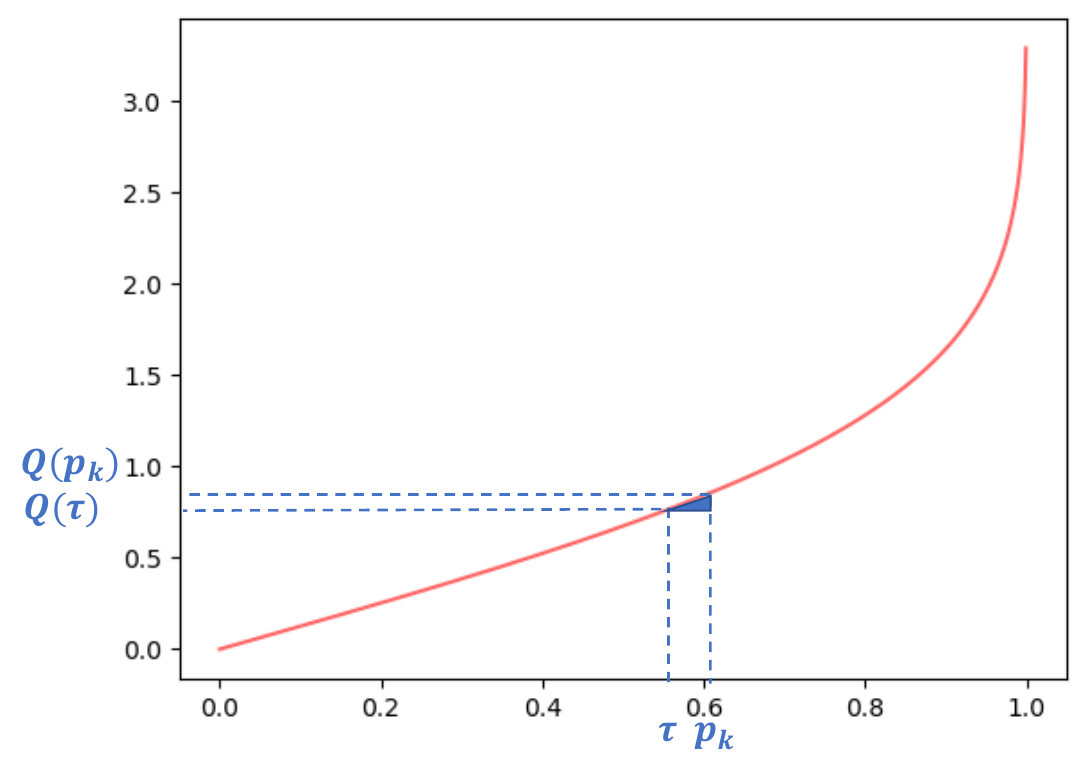}
    \caption{Illustration of ratio between $|Q^{p_k} - Q^{\tau}|$ and $|p_k - \tau|$. The red curve is a toy example quantile function of the absolute Gaussain distribution with location and scale parameters 0 and 1 respectively. }
    \label{fig: theo1 ratio}
\end{figure}



\end{document}